\documentclass[10pt]{article}

\usepackage[preprint]{rlj}   

\usepackage{amssymb,amsfonts,amsthm}
\usepackage{algorithm,algorithmic}
\usepackage{graphicx}
\usepackage{subcaption}
\usepackage{booktabs}
\usepackage{multirow}
\usepackage{float}
\usepackage{xcolor}
\newcommand{\dchg}[1]{\textcolor{blue}{#1}}
\usepackage{tikz}
\usetikzlibrary{positioning, arrows.meta, calc, tikzmark}

\newtheorem{theorem}{Theorem}
\newtheorem{lemma}[theorem]{Lemma}
\newtheorem{proposition}[theorem]{Proposition}

\newtheorem{definition}{Definition}
\newtheorem{remark}{Remark}
\newtheorem{assumption}{Assumption}

\newcommand{\E}{\mathbb{E}}
\newcommand{\R}{\mathbb{R}}
\newcommand{\cS}{\mathcal{S}}
\newcommand{\cA}{\mathcal{A}}

\newcommand{\pib}{\pi_b}

\DeclareMathOperator*{\argmax}{arg\,max}
\DeclareMathOperator*{\argmin}{arg\,min}

\title{Convex-Concave Reinforcement Learning}

\setrunningtitle{Convex-Concave Reinforcement Learning}

\author{Shripad Deshmukh\textsuperscript{1,*} \quad
Yaswanth Chittepu\textsuperscript{1} \quad
Dhawal Gupta\textsuperscript{1, 2} \\
Philip Thomas\textsuperscript{1} \quad 
Scott Niekum\textsuperscript{1}
}

\coverPageAuthor{Shripad Deshmukh \quad Yaswanth Chittepu \quad
Dhawal Gupta \\ Philip Thomas \quad Scott Niekum}

\emails{}

\affiliations{
$^{1}$\textbf{University of Massachusetts} \qquad
$^{2}$\textbf{Institute of Foundation Models}\\[3pt]
{\normalfont\small \textsuperscript{*}Corresponding author:
\texttt{svdeshmukh@cs.umass.edu}}
}


\summary{The optimization problem of maximizing expected return is notoriously non-convex, and the field has largely responded by \emph{avoiding} it: optimizing convex surrogate approximations of the return under trust-region constraints (NPG, TRPO, PPO, AWR).  We show that, in log-density-ratio coordinates $y := \log[\pi/\pi_n]$, the exact per-iteration objective, computable via per-decision importance sampling (PDIS), is a \emph{difference-of-convex program}.  This structure lets us move beyond surrogate approximations and opens a \emph{multi-step} axis $k$ that couples consecutive decisions.  We solve the per-iteration program with sequential convex programming (SCP), the standard solver for difference-of-convex problems, and prove theoretical inner-loop descent and stationarity, with a conditional outer-loop bound. Empirically, multi-step Convex-Concave RL (CCRL) outperforms the convex-surrogate approaches on diagnostic MDPs where credit must propagate across a horizon. Against tuned PPO, it is competitive on classic control and converges markedly faster on a clinically motivated, stochastic, mid-horizon sepsis simulator.}

\contribution{
On-policy methods perform policy search within a trust region of the current policy, typically replacing the non-convex objective of exact return improvement with a convex surrogate. We show that, using per-decision importance sampling (PDIS) and log policy ratios, the exact improvement objective can be written as a difference of convex functions. This formulation recovers convex surrogate methods (CPI, NPG, TRPO, AWR) while extending the policy updates to account for the combined effects of consecutive decisions.
}{
The connections to existing methods concern their surrogate objectives and trust-region constraints, not every detail of their updates. The objective equals the exact return difference at the full horizon, while PDIS truncation trades bias for lower variance.
}

\contribution{
We adapt sequential convex programming (SCP) to policy improvement through tangent-plane steps and normalization that keep theoretical iterates valid. We prove an $O(1/R)$ squared-step bound with stationary limit points under a constraint qualification, and conditional geometric contraction of full-horizon policy updates toward an optimality neighborhood.
}{
The analysis uses full-support tabular policies and exact convex subproblem solves. Outer contraction further requires bounded state coverage and a near-global inner solution. The practical algorithm proposed in the paper is outside these guarantees.
}

\contribution{
CCRL learns faster than tuned convex surrogate methods on delayed-reward diagnostic MDPs, and its advantage grows with the number of dependent decisions. CCRL is competitive with tuned PPO on four classic-control tasks and converges significantly faster on a stochastic healthcare simulator.
}{
All methods are tuned under matched interaction budgets. The classic-control results expose a tradeoff: as $k$ grows, noisy advantage estimates and increasing PDIS variance can outweigh the benefit of coupling more decisions.
}

\newlength{\panelht}

\begin{document}

\makeCover  
\maketitle

\begin{abstract}
Policy learning drives many of the most consequential and heavily-invested applications of reinforcement learning today.  Yet the core optimization problem it rests on (maximizing expected return) is notoriously non-convex, even under a direct policy parameterization, and the field has largely responded by \emph{avoiding} it: optimizing convex surrogate approximations of the return under trust-region constraints (NPG, TRPO, PPO, AWR).  We show that this seemingly unstructured problem is not actually structureless.  In log-density-ratio coordinates $y := \log[\pi/\pi_n]$, the exact per-iteration objective, computable via per-decision importance sampling (PDIS), is a \emph{difference-of-convex-constrained difference-of-convex (DC-constrained DC) program}.  This structure lets us move beyond surrogate approximations: it recovers CPI, NPG, TRPO, and AWR as special cases along interpretable axes, and it opens a \emph{multi-step} axis $k$ that couples consecutive decisions.  We solve the per-iteration program with sequential convex programming (SCP), the standard solver for difference-of-convex problems, and establish inner-loop descent and limit-point stationarity for the theoretical algorithm, together with a conditional outer-loop performance bound.  Empirically, multi-step Convex-Concave RL (CCRL) wins on diagnostic MDPs where credit must propagate across a horizon (its advantage growing with the dependency length), is competitive with a tuned PPO on classic control, and on a clinically motivated, stochastic, mid-horizon sepsis simulator converges markedly faster than tuned PPO to the same near-optimal survival, with an $11.3\%$ higher area under the training curve. The code is publicly available at: \mbox{\url{https://github.com/UMass-SCALAR-Lab/Convex-Concave-RL}}.
\end{abstract}

\vspace{-.5em}
\section{Introduction}
\label{sec:intro}
Learning a good policy is a defining problem for reinforcement learning (RL) researchers and an increasingly consequential one for industrial practitioners. As investment in AI grows, improving how effectively systems learn policies is an important route to stronger capabilities.  The workhorse behind many of RL's recent successes, in robotics and increasingly in large language models, is policy-gradient optimization \citep{schulman2015trpo, schulman2017ppo, shao2024deepseekmath}, and any structural insight into how it works could translate directly into more efficient and reliable learning.

At its core, policy-gradient methods optimize the expected return $J(\pi)$ of a policy $\pi$, an objective that is non-convex even under the direct policy parameterization.  NPG and TRPO use the convex single-step Conservative Policy Iteration (CPI) surrogate \citep{kakade2002cpi,kakade2002npg,schulman2015trpo}; PPO clips a related ratio objective \citep{schulman2017ppo}, while off-policy AWR derives its regression target from a constrained improvement program \citep{peng2020awr}.  This surrogate is a \emph{biased} proxy for $J$, and optimizing it typically improves the policy by considering one decision timestep at a time.  On problems where several timesteps couple to determine the outcome, this single-step approach  can move the policy away from higher-return behavior until later advantage estimates reveal a better direction. As a workaround, the multi-step dependence is instead carried by the value or advantage \emph{estimates} that feed the update, while the policy-optimization objective itself remains a single-step proxy \citep{sutton1988td, schulman2016gae}. In this work we confront this seemingly unyielding non-convex optimization directly, through a structured reformulation of the RL objective.

Our idea is to optimize the exact per-iteration improvement rather than a surrogate of it. The return difference $J(\pi) - J(\pi_n)$ is estimable from on-policy trajectories by per-decision importance sampling (PDIS) \citep{precup2000eligibility}, and is non-convex in $\pi$ only because it contains \emph{products} of importance ratios. Re-expressing each policy through its log-ratio to the current one, $y := \log(\pi/\pi_n)$, converts these products into sums of exponentials and decomposes the objective as a \emph{difference of two convex functions}, a difference-of-convex (DC) program.  We call the resulting framework \emph{Convex-Concave Reinforcement Learning} (CCRL).

The formulation enables three important mechanisms. First, it makes the exact objective directly optimizable, rather than only through a surrogate. Second, it unifies existing methods: at horizon $k = 1$ the program reduces to the single-step convex surrogate optimized by CPI, NPG, TRPO, PPO, and AWR. Third, it exposes the horizon $k$ as a free parameter: $k \geq 2$ couples several consecutive decisions within a single update, so one update can credit an action for outcomes several steps later, and $k = H$, the full horizon, optimizes the exact, unbiased return. We solve this program by sequential convex programming (SCP), the standard method for DC problems~\citep{lethi2018dc} by adapting it to preserve a valid policy each iteration.

In each SCP round, we linearize the concave part of the objective and the policy normalization constraint, solve the resulting convex subproblem, and renormalize its solution to obtain a valid policy. After the final SCP round, the log-density ratios $y^*$ specify the next policy. For the theoretical algorithm, this procedure improves the $k$-step truncated PDIS objective monotonically. When the window spans the full horizon, that objective equals the exact return difference, so the same result gives monotone improvement in expected return. Each round, however, requires a convex solve, which is the computational cost of this approach.

For practical use, we simplify the algorithm at the cost of the theoretical guarantees: we normalize the policy only after the final SCP round and clip log-density ratios to control multi-step importance weights (Section~\ref{sec:practical}). The resulting algorithm retains the multi-step objective, allowing each update to account for how consecutive decisions jointly affect return. In the risky corridor, CCRL learns faster than the tuned convex surrogate approaches (TRPO, AWR, and PPO); in increasingly deep risky chains, its advantage over PPO grows with dependency length. On four classic-control tasks, it is competitive with tuned PPO, although longer windows add variance without a matching credit-assignment gain. On a stochastic, mid-horizon sepsis simulator, CCRL converges faster than tuned PPO, with an $11.3\%$ higher area under the training curve.

\noindent\textbf{Contributions.}
\begin{enumerate}[label=\textbf{(\arabic*)}]
\item We formulate on-policy improvement within a trust region around the sampling policy as a difference-of-convex (DC) program in log-sum-exp functions. At the full horizon, its objective equals the exact return difference; at a single step, it recovers the CPI surrogate underlying NPG, TRPO, PPO, and AWR. Between these cases, the formulation lets one policy update couple consecutive decisions (Sections~\ref{sec:dc} and~\ref{sec:special_cases}).
\item We adapt sequential convex programming (SCP) to policy improvement, using a tangent-plane constraint and normalization to keep each theoretical iterate a valid policy. For the tabular algorithm, we bound the smallest squared SCP step by $O(1/R)$ and prove stationarity of limit points under a constraint qualification. Under additional coverage and inner-solution conditions, full-horizon policy updates contract geometrically toward a neighborhood of the optimum (Sections~\ref{sec:algorithm} and~\ref{sec:convergence}).
\item On a diagnostic delayed-reward MDP, CCRL learns faster than tuned convex surrogate baselines (TRPO, AWR, and PPO); in deeper chains, its advantage over PPO grows as more decisions must be coordinated. It is competitive with tuned PPO on four classic-control tasks and converges faster on a stochastic healthcare simulator (Sections~\ref{sec:practical} and~\ref{sec:experiments}).
\end{enumerate}

The difference-of-convex view of policy optimization offers a structural handle on an objective long treated as intractable, and a principled route to policy-gradient methods that exploit multi-step dependencies for more efficient training.

\section{Preliminaries}
\label{sec:prelim}
\noindent\textbf{Markov decision process.} For a set $\mathcal{X}$, let $\Delta(\mathcal{X})$ denote the set of probability distributions over $\mathcal{X}$.  We work with a finite-horizon discounted MDP \citep{puterman2014markov} of horizon $H < \infty$, given by the tuple $(\cS, \cA, p, r, \gamma, d_0)$: a finite state set $\cS$, a finite action set $\cA$, a discount factor $\gamma \in [0,1]$, and an initial-state distribution $d_0 \in \Delta(\cS)$.  Capital letters denote random variables: $S_t \in \cS$ and $A_t \in \cA$ are the state and action at time $t$, and $R_{t+1} \in \R$ the reward that follows.  The transition function $p(s' \mid s, a) := \Pr(S_{t+1} = s' \mid S_t = s, A_t = a)$ gives the probability of moving to $s'$, and the reward function $r(s, a) := \E[R_{t+1} \mid S_t = s, A_t = a]$ the expected immediate reward.  As is standard in the finite-horizon setting, the state includes the time index. A policy $\pi : \cS \to \Delta(\cA)$ acts by $\pi(a \mid s) := \Pr(A_t = a \mid S_t = s)$ and induces a distribution over trajectories $\tau = (S_0, A_0, \ldots, S_H)$, $p_\pi(\tau) = d_0(S_0) \prod_{t=0}^{H-1} \pi(A_t \mid S_t)\, p(S_{t+1} \mid S_t, A_t)$. Write $G_t := \sum_{j=t}^{H-1} \gamma^{j-t} R_{j+1}$ for the discounted return from time $t$.  We define the state-value $V^\pi(s) := \E_\pi[G_t \mid S_t = s]$, action-value $Q^\pi(s, a) := \E_\pi[G_t \mid S_t = s, A_t = a]$, and advantage $A^\pi(s, a) := Q^\pi(s, a) - V^\pi(s)$. The policy performance is $J(\pi) := \E_{S_0 \sim d_0}[V^\pi(S_0)]$.  We will also use the discounted state-visitation \emph{distribution} $d^\pi(s) := \eta^\pi(s)/Z_H$, where $\eta^\pi(s) := \sum_{t=0}^{H-1} \gamma^t \Pr(S_t = s \mid \pi)$ is the un-normalized state-visitation count (the expected discounted number of visits to $s$ under $\pi$), and $Z_H := \sum_{t=0}^{H-1} \gamma^t$ is its total mass.

\noindent\textbf{Iterative on-policy improvement.} Our setting is iterative on-policy improvement: at every iteration $n$, the algorithm samples a batch of trajectories under the current policy $\pi_n$ and uses those trajectories to produce an improved candidate $\pi$, which then becomes $\pi_{n+1}$.  The objective at each iteration is to maximize $J(\pi)$, or equivalently (since $J(\pi_n)$ is constant), the \emph{one-step improvement} $J(\pi) - J(\pi_n)$.  Because the candidate must be scored from $\pi_n$-data, the search is restricted to a neighborhood of $\pi_n$, where data-based estimates of $J(\pi) - J(\pi_n)$ are reliable only when $\pi$ is close to $\pi_n$. The performance-difference lemma \citep{kakade2002cpi} expresses the improvement $J(\pi) - J(\pi_n)$ as an expectation of $\pi_n$'s advantages under the candidate policy's distribution, in two equivalent forms:
\[
  J(\pi) - J(\pi_n)
  \;=\; \underbrace{Z_H \cdot \E_{S \sim d^\pi}\!\Bigl[\sum_a \pi(a|S)\, A^{\pi_n}(S, a)\Bigr]}_{\text{state-visitation form}}
  \;=\; \underbrace{\E_{\tau \sim p_\pi}\!\Bigl[\, \sum_{t=0}^{H-1} \gamma^t\, A^{\pi_n}(S_t, A_t) \,\Bigr]}_{\text{trajectory form}}.
\]
Both forms require expectations under the candidate policy's distribution ($d^\pi$ or $p_\pi$), which are unavailable, whereas the data at hand is collected under $\pi_n$. There are broadly two routes forward.

\emph{Route 1: Distribution substitution (biased).} We can substitute $d^{\pi_n}$ for $d^\pi$ in the state-visitation form, yielding the data-estimable objective $L(\pi; \pi_n) := \E_{S \sim d^{\pi_n}}\!\bigl[\sum_a \pi(a|S)\, A^{\pi_n}(S, a)\bigr]$. This is a biased proxy for $J(\pi) - J(\pi_n)$, with bias that is small when $\pi$ is close to $\pi_n$ and grows with the divergence. The substituted objective $L$ is precisely the CPI surrogate of \citet{kakade2002cpi}, and the basis for the optimization target of NPG \citep{kakade2002npg}, TRPO \citep{schulman2015trpo}, and PPO \citep{schulman2017ppo}.

\emph{Route 2: Distribution correction.} We can correct the mismatch exactly by reweighting trajectories with importance-sampling ratios.  With single-step IS ratio $\rho(s, a) := \pi(a \mid s) / \pi_n(a \mid s)$, the trajectory form becomes an \emph{exact} expectation under $p_{\pi_n}$:
\begin{equation}
  J(\pi) - J(\pi_n)
  \;=\; \E_{\tau \sim p_{\pi_n}}\!\left[
      \sum_{t=0}^{H-1} \gamma^t
      \Bigl(\prod_{t'=0}^{t} \rho(S_{t'}, A_{t'})\Bigr)
      A^{\pi_n}(S_t, A_t)
    \right].
  \label{eq:pdis}
\end{equation}
This is the per-decision importance-sampling (PDIS) identity \citep{precup2000eligibility}: the cumulative product is an exact distribution correction from $p_{\pi_n}$ to $p_\pi$, recovering the true improvement with no bias, at the cost of variance that grows exponentially with the horizon. To control the variance, we can use two devices: (i) truncating the product to a $k$-step window (adding bias), (ii) constraining $\pi$ to a trust region around $\pi_n$.

\noindent\textbf{$k$-step PDIS.} Truncating the cumulative product in~\eqref{eq:pdis} to the most recent $k$ steps gives
\begin{equation}
  J_k(\pi; \pi_n) := \E_{\tau \sim p_{\pi_n}}\!\left[
      \sum_{t=0}^{H-1} \gamma^t
      \Bigl(\prod_{t'=\max(0,\,t-k+1)}^{t} \rho(S_{t'}, A_{t'})\Bigr)
      A^{\pi_n}(S_t, A_t)
    \right].
  \label{eq:pdis_k}
\end{equation}
At $k = 1$ the product collapses to the single-step ratio and $J_1 = Z_H \E_{d^{\pi_n}}[\cdot] = Z_H\, L$ (Route 1); at $k = H$ no truncation occurs and $J_H = J(\pi) - J(\pi_n)$ (Route 2).  For $1 < k < H$, $k$ controls a bias-variance tradeoff.

\noindent\textbf{Trust-region constrained problem.} Since $J_k(\pi; \pi_n)$ is a reliable surrogate only when $\pi$ stays close to $\pi_n$, we must explicitly constrain the candidate to a neighborhood of $\pi_n$.  TRPO uses an explicit KL trust-region constraint \citep{schulman2015trpo}, whereas PPO uses a clipped, trust-region-inspired surrogate \citep{schulman2017ppo}. Here we constrain an $f$-divergence ball $D_f(\pi \,\|\, \pi_n) \le \delta$ \citep{belousov2017f}, where $\delta > 0$ is the trust-region size. The divergence $D_f$ can be instantiated in several ways, which we cover in Table~\ref{tab:trust_region_variants} (Section~\ref{sec:trust_logratio}). The iteration's problem reads
\begin{equation}
  \max_\pi \; J_k(\pi; \pi_n)
  \qquad \text{subject to} \qquad
  D_f(\pi \,\|\, \pi_n) \leq \delta.
  \label{eq:goal}
\end{equation}
\vspace{-1.6em}
\subsection{Non-convexity of general iterative on-policy optimization}
\label{sec:nonconvex}
The iterative on-policy improvement problem~\eqref{eq:goal} is generally non-convex in the policy $\pi$.  For instance, take the \emph{direct parameterization} of $\pi$, where every state-action pair $(s, a)$ has a separate parameter $\pi(a|s)$.  The feasible set here is convex; the space of valid policies $\Pi$ is a product of probability simplices, and the $f$-divergence trust region $D_f(\pi \,\|\, \pi_n) \le \delta$ is convex in $\pi$ as long as the generator $f$ is convex \citep{belousov2017f}.  The non-convexity lies in the optimization objective $J_k(\pi; \pi_n)$.  At $k = 1$, $J_1$ is linear in $\pi$ and problem~\eqref{eq:goal} is convex; but for $k \ge 2$, the cumulative IS product $\prod_{t'} \rho(S_{t'}, A_{t'})$ is a product of linear functions of $\pi$ across distinct $(s, a)$ pairs along a trajectory, and this multiplicative coupling makes $J_k$ non-convex in $\pi$.  Gradient-based optimization techniques applied to problem~\eqref{eq:goal} cannot guarantee global optimality of the solution obtained.

\emph{The optimization problem being non-convex does not mean it is completely structureless, however.}  Non-convex functions can sometimes be re-cast in forms that enable more effective optimization than generic gradient methods.  One such structure is the \emph{difference-of-convex} (DC) decomposition: a function is DC if it can be written as $g - h$ with both $g$ and $h$ convex. DC programs are well-studied in the optimization theory literature \citep{tao1997convex}. Finding a global optimum of a difference-of-convex program is NP-hard in general. The structure nevertheless suggests a useful iteration: linearizing one convex component at the current iterate leaves a \emph{convex} subproblem. Under suitable conditions, solving these subproblems gives descent and stationary limit points, without guaranteeing a global optimum. In the next section, we show that under a suitable change of variable, the trust-region PDIS problem~\eqref{eq:goal} admits an explicit DC decomposition.

\section{Iterative On-Policy Improvement as a Difference-of-Convex Program}
\label{sec:dc}

The $k$-step PDIS objective $J_k(\pi; \pi_n)$ is non-convex in $\pi$ because of the multiplicative IS product across distinct $(s, a)$ pairs (Section~\ref{sec:nonconvex}).  Applying a logarithmic change of variable---a classical geometric-programming technique \citep{boyd2004convex} that turns products into sums inside an exponential---yields an explicit \emph{difference-of-convex} (DC) decomposition for every $k$ (Theorem~\ref{thm:dc}).  This DC structure underlies our algorithm and convergence analysis (Sections~\ref{sec:algorithm}, \ref{sec:convergence}) and recovers CPI, TRPO, and AWR as $(k, \alpha)$ special cases (Section~\ref{sec:special_cases}).

\vspace{-.75em}
\subsection{DC structure of the $k$-step PDIS objective}
\label{sec:dc_decomp}
Let \(y(s,a) := \log({\pi(a|s)}/{\pi_n(a|s)}) \in \R\) denote the \emph{log-density ratio} and, along a trajectory $\tau$, write
\begin{equation}
  Y^k_t \;:=\; \sum_{t'=\max(0,\,t-k+1)}^{t} y(S_{t'}, A_{t'})
  \label{eq:window_sum}
\end{equation}
for its sum over the $k$-step window ending at $t$, and abbreviate $A^{\pi_n}_t := A^{\pi_n}(S_t, A_t)$.  Then $\rho(s,a) = e^{y(s,a)}$ and $\pi(a|s) = \pi_n(a|s)\, e^{y(s,a)}$, so the IS product over that window is $\prod_{t'} \rho(S_{t'}, A_{t'}) = e^{Y^k_t}$, a convex function of $y$, and the PDIS objective in $y$-coordinates reads $J_k(y) = \E_{\tau \sim p_{\pi_n}}\bigl[\sum_{t=0}^{H-1} \gamma^t e^{Y^k_t} A^{\pi_n}_t\bigr]$, a weighted sum of convex exponentials with weights $\gamma^t A^{\pi_n}_t$ of either sign.  Writing $(x)_+ := \max(x,0)$ and $(x)_- := \max(-x,0)$ for the positive and negative parts, gathering the terms by the sign of the advantage yields a difference-of-convex decomposition.

\begin{theorem}[DC structure of the $k$-step PDIS objective]
  \label{thm:dc}
For any $k \geq 1$, the $k$-step PDIS objective decomposes as $J_k(y) = J_k^+(y) - J_k^-(y)$, where $J_k^+$ and $J_k^-$ collect the terms with positive and with negative advantage:
  \begin{alignat}{2}
    J_k^{+}(y) &:= \E_{\tau \sim p_{\pi_n}}\Bigl[
      \sum_{t=0}^{H-1} \gamma^t e^{Y^k_t} (A^{\pi_n}_t)_+\Bigr],
    &\qquad
    J_k^{-}(y) &:= \E_{\tau \sim p_{\pi_n}}\Bigl[
      \sum_{t=0}^{H-1} \gamma^t e^{Y^k_t} (A^{\pi_n}_t)_-\Bigr].
    \label{eq:jpm}
  \end{alignat}
Both are non-negative-weighted sums of convex exponentials, hence convex (proof in Appendix~\ref{app:proof_thm_dc}).
\end{theorem}
\vspace{-1em}
\subsection{The space of valid log-density ratios $y(s,a)$}
\label{sec:validity}

Any policy $\pi$ defined via $\pi(a|s) = \pi_n(a|s)\, e^{y(s,a)}$ must satisfy the per-state probability simplex constraint.  That is, for each state $s \in \cS$, $\pi(a|s) \ge 0$ for all $a \in \cA$ and $\sum_a \pi(a|s) = 1$.  Of these, the non-negativity is implicitly handled due to $e^{y(s,a)} > 0$.  For the normalization, define the \emph{per-state normalizer}
\begin{equation}
  f_s(y) \;:=\; \sum_{a \in \cA} \pi_n(a|s)\, e^{y(s,a)},
  \qquad s \in \cS,
  \label{eq:fs}
\end{equation}
so that $\sum_a \pi(a|s) = f_s(y)$, and the normalization constraint is exactly $f_s(y) = 1$ at every state.  The set of valid log-density ratios is therefore the \emph{validity manifold}
\begin{equation}
  \mathcal{M} \;:=\; \bigl\{\, y \in \R^{|\cS||\cA|} \;:\;
    f_s(y) = 1, \;\; \forall\, s \in \cS \,\bigr\}.
  \label{eq:manifold}
\end{equation}
Each $f_s$ is convex in $y$ (a non-negative weighted sum of exponentials of affine maps), so the equality $f_s(y) = 1$ is a DC constraint: the inequality $f_s(y) \le 1$ is convex, while $f_s(y) \ge 1$ is reverse-convex.  The set $\mathcal{M}$ is a smooth embedded manifold of dimension $|\cS|(|\cA| - 1)$: the constraint gradients $\nabla f_s(y)$ have disjoint per-state supports and are never zero (each has strictly positive entries within its state block), so they are linearly independent at every $y$, and the implicit function theorem applies.

\vspace{-.5em}
\subsection{Representing $f$-divergence trust region constraints in terms of log-density ratios}
\label{sec:trust_logratio}
\vspace{-.5em}
The trust-region constraint in problem~\eqref{eq:goal} is written in terms of the candidate policy $\pi$. Now that we optimize over log-density ratios $y$, we must express that constraint in $y$ as well. A general $f$-divergence does not survive this substitution in a usable form. We therefore work with \emph{Amari's $\alpha$-divergence} family \citep{amari1985differential}: for distributions $p$ and $q$ on a common domain and a parameter $\alpha \in \R \setminus \{0, 1\}$,
\begin{equation}
  D_\alpha(p \| q) := \frac{1}{\alpha(\alpha-1)} \Bigl(\E_{X \sim q}\bigl[(p(X)/q(X))^\alpha\bigr] - 1\Bigr).
  \label{eq:alpha_div}
\end{equation}
The ratio enters as a power $(\pi/\pi_n)^\alpha$, so the same logarithmic change of variable turns it into $e^{\alpha\, y}$ (Section~\ref{sec:dc_decomp}), keeping the trust region in the exponential-of-affine form the DC analysis needs.  While it does not represent every possible divergence, specific choices of $\alpha$ recover the ones commonly used in RL, which we list in Table~\ref{tab:alpha_div} in Appendix~\ref{app:fdiv_convex}.

Beyond the choice of divergence, there are three ways to apply the trust-region bound. A \emph{trajectory-level} constraint bounds the divergence between the full trajectory distributions $p_\pi$ and $p_{\pi_n}$. A \emph{strict per-state} constraint bounds the divergence between $\pi(\cdot|s)$ and $\pi_n(\cdot|s)$ separately at every state $s$. An \emph{expected per-state} constraint bounds the average of these per-state divergences over $S \sim d^{\pi_n}$. Under the substitution $\pi(a|s) = \pi_n(a|s)\, e^{y(s,a)}$, all three forms involve exponentials of affine functions of $y$. Exact $f$-divergence forms of these three approaches and their corresponding $\alpha$-divergence forms are provided in Table~\ref{tab:trust_region_variants}.

\begin{proposition}[Trust-region convexity, finite $\alpha$]
  \label{prop:trust}
For every finite $\alpha \in \R \setminus [0,1]$, the trajectory-level, strict per-state, and expected per-state $\alpha$-divergence constraints are convex in $y$ on all of $\R^{|\cS||\cA|}$.
\end{proposition}

The proof is in Appendix~\ref{app:fdiv_convex}. The KL endpoints behave differently: reverse KL $(\alpha \rightarrow 0^-)$ is convex in $y$, while the forward-KL limit $(\alpha \rightarrow 1^+)$ is finite only on the validity manifold $\mathcal{M}$ and its off-manifold expression is not convex; the appendix gives the details. In practice we use a finite $\alpha \notin [0,1]$ ($\chi^2$, $\alpha = 2$), globally convex in $y$.
\vspace{-.5em}
\subsection{Full iterative on-policy improvement problem in log-density ratios}
\label{sec:full_dc_program}
\vspace{-.5em}
We now combine the DC objective (Theorem~\ref{thm:dc}) defined over the validity manifold under a trust-region constraint form from Table~\ref{tab:trust_region_variants} (with finite $\alpha \notin [0,1]$) into a single optimization problem at each policy iteration:
\begin{equation}
  {\setlength{\fboxrule}{1pt}\setlength{\fboxsep}{8pt}%
  \fcolorbox{customBoxColor!75!black}{customBoxColor}{$\displaystyle
  \begin{aligned}
    \max_{y \in \R^{|\cS||\cA|}} \quad
      & J_k^+(y) - J_k^-(y)
        & & \text{(DC objective)} \\
    \text{s.t.} \quad
      & D_\alpha(\pi \,\|\, \pi_n) \leq \delta
        & & \text{(trust region, convex)} \\
      & f_s(y) = 1, \quad \forall\, s \in \cS
        & & \text{(validity, DC)}
  \end{aligned}
  $}}
  \label{eq:ccrl_full}
\end{equation}
We call this the \textbf{\emph{Convex-Concave Reinforcement Learning (CCRL)}} formulation: a DC objective under a DC (validity) equality and a convex (trust-region) inequality.
\vspace{-.5em}
\subsection{Relationship between CCRL and the prevalent trust-region based on-policy algorithms}
\label{sec:special_cases}
\vspace{-.5em}
CPI \citep{kakade2002cpi}, NPG \citep{kakade2002npg}, TRPO \citep{schulman2015trpo}, and the off-policy AWR \citep{peng2020awr} are related through single-step policy-improvement programs.

\begin{proposition}[Recovery of single-step trust-region programs]
  \label{prop:cpi}
At $k=1$, the CCRL objective equals the CPI surrogate up to a positive constant: $J_1(y)=Z_H L(\pi;\pi_n)$. With expected per-state reverse-KL and forward-KL trust regions, respectively, this yields the policy-space constrained improvement programs underlying TRPO and AWR.
\end{proposition}

Because $Z_H>0$, the scaling does not change the maximizers under a given constraint set. TRPO and AWR both use expected per-state KL trust regions, but in opposite directions. TRPO uses $\E_{S\sim d^{\pi_n}}[\mathrm{KL}(\pi_n(\cdot|S)\,\|\,\pi(\cdot|S))]\leq\delta$, the reverse-KL endpoint of the $\alpha$-divergence family. NPG shares the CPI objective and this reverse-KL geometry, implementing the update through a local Fisher-matrix approximation in parameter space. AWR uses $\E_{S\sim d^{\pi_n}}[\mathrm{KL}(\pi(\cdot|S)\,\|\,\pi_n(\cdot|S))]\leq\delta$, the forward-KL endpoint. In its single-sampling-policy setting, the resulting policy-space problem has the normalized exponentiated-advantage solution when the KL constraint is active. We establish the AWR correspondence in policy space because the forward-KL expression is not globally convex in $y$. Derivations are in Appendices~\ref{app:proof_prop_cpi} and~\ref{app:proof_prop_awr}.

Thus, the CCRL framework enables unified treatment of these existing approaches. More importantly, it extends beyond them along two axes: increasing $k$ allows us to incorporate longer temporal dependence between decisions into the policy optimization and varying $\alpha$ enables a choice among different divergences.

\vspace{-1em}
\section{Policy optimization using Sequential Convex Programming (SCP)}
\label{sec:algorithm}
\vspace{-.5em}

\paragraph{Solving the DC-constrained DC program.}
We solve problem~\eqref{eq:ccrl_full} by minimizing $\Phi:=J_k^- - J_k^+$ over the trust region and the validity manifold $\mathcal M$. Its DC structure suggests sequential convex programming (SCP) \citep{tao1997convex,lipp2016variations}: linearize $J_k^+$ at the current iterate and solve the resulting convex subproblem. The validity equality requires separate treatment.

\noindent\textbf{The textbook treatment.} Textbook SCP writes $f_s(y)=1$ as two inequalities. It retains the convex constraint $f_s(y)\leq1$ and linearizes $f_s(y)\geq1$ at a valid iterate $y^r$, imposing $\nabla f_s(y^r)^\top(y-y^r)\geq0$.
\begin{proposition}[Degeneracy of the textbook validity split]
  \label{prop:naive_singleton}
If $f_s(y^r)=1$ for every state, then
\mbox{\(
  \bigl\{y:f_s(y)\leq1,\ \nabla f_s(y^r)^\top(y-y^r)\geq0\ \forall s\bigr\}=\{y^r\}.
\)}
\end{proposition}
The proof is provided in Appendix~\ref{app:linearize_equality}. The textbook subproblem therefore cannot produce a policy update, motivating the tangent-plane construction below.

\noindent\textbf{Our fix: tangent-plane SCP with retraction.} Instead of splitting $f_s(y)=1$ into two inequalities, we impose its tangent-plane equality at the current valid iterate $y^r$. This gives the subproblem room to move. Its solution $\hat y^{r+1}$ need not define a valid policy, so we renormalize it per state to obtain $y^{r+1}\in\mathcal M$. We also add a proximal term with weight $\rho>0$ to control the move and a changing lower bound $b_r$ to maintain policy support. Writing $\tilde\Phi_r(y):=J_k^-(y)-\nabla J_k^+(y^r)^\top y$, one round is given by Steps 1--2 below.
\begin{align}
  \text{(Step 1)}\quad
  \hat y^{r+1}
  &\;=\; \argmin_{y \in \mathcal{C}_r}\;
        \Bigl( \tilde\Phi_r(y) \,+\, \tfrac{\rho}{2}\|y - y^r\|^2 \Bigr),
  \label{eq:scp_sub} \\
  \mathcal{C}_r
  &:= \bigl\{\, y \;:\;
       D_\alpha(y) \leq \delta,\;
       \nabla f_s(y^r)^\top (y - y^r) = 0 \;\forall s,\;
       y \geq b_r
     \,\bigr\},
  \label{eq:Cr} \\
  \text{(Step 2)}\quad
  y^{r+1}(s,a)
  &\;:=\; \hat y^{r+1}(s,a) - \log f_s(\hat y^{r+1}),
  \label{eq:retraction}
\end{align}
Algorithm~\ref{alg:ccrl} in Appendix~\ref{app:algorithm} gives the full policy iteration, including advantage estimation.
Step 2 is a per-state softmax in policy space, $\pi_{r+1}(a|s)\propto\pi_n(a|s)e^{\hat y^{r+1}(s,a)}$. Because $\hat y^{r+1}$ lies in the tangent plane at $y^r$, the normalization shift $c_s:=\log f_s(\hat y^{r+1})$ is nonnegative and second order in the step size (Lemma~\ref{lem:retraction}). For $\alpha>1$, subtracting $c_s$ cannot increase $D_\alpha$, so the retracted iterate remains inside the trust region (Lemma~\ref{lem:R1}). The proximal term provides enough decrease in $\Phi$ to absorb the change caused by normalization when $\rho$ is sufficiently large (Section~\ref{sec:inner_convergence}).

We initialize the support floor at $b_0(s,a):=\log(p_{\min}/\pi_n(a|s))$ and, after each round, set $b_{r+1}:=b_r-(\max_s c_s)\mathbf1$. This makes the retracted iterate feasible for the next subproblem. Together with the strict per-state trust region used in the convergence analysis, the floor keeps the subproblem compact, while the descent bound limits its total relaxation (Appendix~\ref{app:proofs_convergence}).

\noindent\textbf{Recovering the policy.} After $R$ rounds, the final iterate $y^R$ satisfies $f_s(y^R) = 1$ exactly and defines the next policy directly:
\begin{equation}
  \pi_{n+1}(a|s) \;=\; \pi_n(a|s)\, e^{y^R(s,a)},
  \label{eq:recovery_exact}
\end{equation}
a valid distribution at every state.  For parametric policy classes, the right-hand side is the target distribution to which $\pi_\theta$ is fit by maximum likelihood (Section~\ref{sec:practical}).

Starting from $y^0=0$ (the current policy), we run $R$ rounds of Steps 1--2 and recover $\pi_{n+1}$ by~\eqref{eq:recovery_exact}.

\vspace{-1em}
\section{Convergence Analysis}
\label{sec:convergence}
\vspace{-.5em}
We analyze the theoretical tabular algorithm at two levels.  Within a policy iteration, we show that the SCP rounds decrease the objective, bound a proximal-step residual, and reach d-stationary limit points under a constraint qualification.  Across policy iterations, when the window spans the full horizon, we establish a conditional contraction of the performance gap toward a neighborhood of the optimum.  The assumptions and precise bounds are stated below; proofs are in Appendix~\ref{app:proofs_convergence}.

\vspace{-.5em}
\subsection{Inner-loop convergence : SCP convergence within a policy iteration}
\label{sec:inner_convergence}
\vspace{-.5em}

The proof rests on controlling the effect of retraction. The convex SCP subproblem decreases $\Phi$ by an amount proportional to the squared step $\|\hat y^{r+1}-y^r\|^2$. Retraction restores a valid policy and changes the solution by only a second-order amount in that step. With a sufficiently large proximal weight $\rho$, the decrease therefore survives retraction. Summing these decreases bounds both the steps and the total relaxation of the moving support floor.

The analysis uses a strict per-state trust region and the moving floor $b_r$ defined in Eq.~\eqref{eq:floor_schedule}, beginning at $b_0(s,a)=\log[p_{\min}/\pi_n(a|s)]$. Write $\hat z_r=\hat y^{r+1}-y^r$ for the step taken by the convex subproblem, before retraction. A point is directionally stationary (d-stationary) if it has no feasible first-order descent direction. The stationarity conclusion uses the linear independence constraint qualification (LICQ), which requires the gradients of the active constraints to be linearly independent.

\begin{theorem}[SCP convergence within a policy iteration]
  \label{thm:inner}
Fix a full-support tabular policy $\pi_n$, a strict per-state $\alpha$-divergence trust region with $\alpha>1$ and $\delta>0$, and a support level $0<p_{\min}\leq\min_{s,a}\pi_n(a|s)$ with $|\cA|p_{\min}<1$. Run the exact SCP subproblem and retraction with the moving floor described above. If the proximal weight $\rho$ is sufficiently large, then:
  \begin{enumerate}[leftmargin=2em,itemsep=1pt,label=(\roman*)]
\item Every SCP round is well-defined. Its retracted iterate is a valid policy within the trust region, and the moving floor converges to a finite limit $b_\infty$.
\item There are constants $c,C>0$ such that $\Phi(y^{r+1})\leq\Phi(y^r)-c\|\hat z_r\|^2$ and $\min_{r<R}\|\hat z_r\|^2\leq C/R$. In particular, $\hat z_r\to0$.
\item Let $\mathcal F_\infty$ be the set of $y$ satisfying $f_s(y)=1$ and $D_\alpha^s(y)\leq\delta$ for every state $s$, and $y\geq b_\infty$. Every limit point is d-stationary for $\Phi$ over $\mathcal F_\infty$ provided its active constraints satisfy LICQ.
  \end{enumerate}
\end{theorem}

The descent bound implies that the smallest squared step among the first $R$ rounds is $O(1/R)$, and that $\hat z_r\to0$. The iterates have limit points because they remain in a fixed compact set. At any such point satisfying LICQ, the limiting optimality conditions of the convex subproblems rule out a feasible first-order descent direction for $\Phi$ on $\mathcal F_\infty$. This is a limit-point stationarity guarantee; it does not assert convergence of the entire sequence or global optimality of the inner problem.

\vspace{-.5em}
\subsection{Outer-loop convergence: Policy iterations leading to an optimality neighborhood}
\label{sec:outer_convergence}
\vspace{-.5em}
We now consider the policies produced across iterations. At $k=H$, the PDIS objective equals the exact performance difference, assuming exact expectations and advantages. The proof compares each update with a feasible policy formed by mixing the current policy with a greedy policy that retains a small probability for every action. This comparison gives an improvement proportional to the current optimality gap, with an error caused by the support floor. An update whose inner objective is close to its global minimum inherits that improvement up to its inner-solve gap.

Write $\Delta_n:=J(\pi^*)-J(\pi_n)$ for the performance gap and $\epsilon_n$ for the difference between the returned policy's inner-objective value and the global minimum of that iteration's constrained problem. Let $\bar\epsilon:=\sup_n\epsilon_n$. The following result uses strict per-state $\chi^2$ trust regions ($\alpha=2$).

\begin{theorem}[Outer-loop contraction]
  \label{thm:outer}
Suppose every outer policy assigns at least $p_{\min}>0$ to every action, with $|\cA|p_{\min}<1$, and the state-visitation distribution of $\pi^*$ is uniformly covered by those of the outer policies. If $\bar\epsilon$ is finite, there are constants $c\in(0,1)$ and $\beta\geq0$ such that $J(\pi_{n+1})-J(\pi_n)\geq c\Delta_n-\beta-\epsilon_n$. Consequently, $\Delta_N\leq(1-c)^N\Delta_0+(\beta+\bar\epsilon)/c$.
\end{theorem}

Thus the performance gap approaches a neighborhood of radius $(\beta+\bar\epsilon)/c$. The floor accounts for $\beta$; the global inner-solve gap accounts for $\bar\epsilon$. These are separate conditions: the inner-loop theorem establishes neither a common support floor across policy iterations nor a small global inner-solve gap. The constants and proof are given in Appendix~\ref{app:proofs_convergence}.

\vspace{-1em}
\section{Practical implementation of CCRL}
\label{sec:practical}
\vspace{-.5em}

For the experiments, we estimate the advantages and the expectations in $J_k^\pm$ from trajectories collected under the current policy. Starting at $y^0=0$, each SCP round solves the following convex program, with $J_k^\pm$ denoting these empirical objectives:
\begin{equation}
  y^{r+1} = \argmin_{y}
  \bigl( J_k^-(y) - \nabla J_k^+(y^r)^\top y
    + \tfrac{\rho}{2}\|y - y^r\|^2 \bigr)
  \ \ \text{s.t.}\ \ D_\alpha(y) \leq \delta,\ \
  \|y\|_\infty \leq y_{\max},
  \label{eq:scp_sub_practical}
\end{equation}
The program retains the $k$-step objective and trust-region constraint, but omits the tangent-plane validity constraint, per-round retraction, and moving support floor of the theoretical algorithm. We normalize once, after the final SCP round, to form the policy target $\tilde\pi_{n+1}(a|s)\propto\pi_{\theta_n}(a|s)e^{y^R(s,a)}$.

The constraint $\|y\|_\infty\leq y_{\max}$ bounds each log-density ratio used in the SCP rounds. A product spanning at most $k$ decisions is consequently bounded by $e^{k y_{\max}}$. The clip controls the magnitude of the importance weights, while a tighter clip also restricts which policy updates the optimizer can find (Appendix~\ref{app:pdis-variance}). We fit the parametric policy $\pi_\theta$ to the normalized target by maximum likelihood. These changes simplify the update used in the experiments, but the guarantees of Section~\ref{sec:convergence} apply to the theoretical tabular algorithm. In particular, descent of the unnormalized SCP objective does not establish monotone return improvement after normalization and policy fitting.

\vspace{-.75em}
\section{Experiments}
\label{sec:experiments} 
\vspace{-.75em}

We evaluate CCRL against strong on-policy baselines to answer four questions.  (i)~Does multi-step PDIS ($k \geq 2$) recover credit that the single-step CPI surrogate ($k = 1$) misses (Section~\ref{sec:risky-corridor})?  (ii)~Do those gains amplify as the dependency horizon deepens (Section~\ref{sec:risky-chain})? (iii)~How does CCRL compare against strong baselines on standard classic-control benchmarks (CartPole, Acrobot, MountainCar, LunarLander; Section~\ref{sec:classic-control})?  (iv)~Does the multi-step objective help on a clinically motivated, stochastic task (Section~\ref{sec:sepsis})?  Every method is tuned independently under matched per-iteration budgets.

\vspace{-.75em}
\subsection{Multi-step credit propagation via $k > 1$ truncated PDIS}
\label{sec:risky-corridor}
\vspace{-.5em}
\begin{figure}[h]
    \centering
    \settoheight{\panelht}{\includegraphics[width=0.33\textwidth]{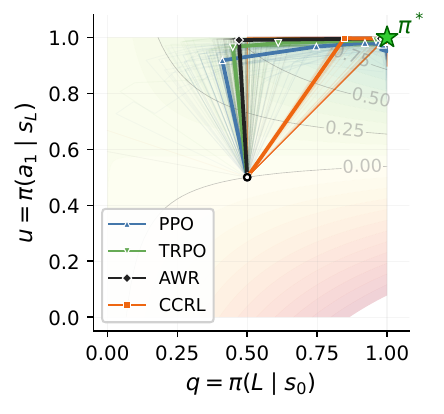}}%
    \begin{subfigure}[b]{0.33\textwidth}
        \centering
        \parbox[c][\panelht][c]{\linewidth}{\centering
        \resizebox{\linewidth}{!}{%
        \begin{tikzpicture}[
            >={Stealth[length=2mm]},
            mdp/.style={circle, draw, thick, minimum size=8mm, inner sep=0pt, font=\Large},
            lbl/.style={font=\large, inner sep=1.5pt},
          ]
          \node[mdp, fill=blue!15]  (s0)  at (0,0)       {$s_0$};
          \node[mdp, fill=red!15]   (sL)  at (-1.7,-1.7) {$s_L$};
          \node[mdp, fill=green!15] (sR)  at (1.7,-1.7)  {$s_R$};
          \node[mdp, fill=black!8]  (tL1) at (-2.6,-3.4) {$s_T$};
          \node[mdp, fill=black!8]  (tL2) at (-0.9,-3.4) {$s_T$};
          \node[mdp, fill=black!8]  (tR)  at (1.7,-3.4)  {$s_T$};
          \draw[->, thick, red!70!black]   (s0) -- (sL)  node[lbl, midway, above left]  {$L,\ r{=}{-}20$};
          \draw[->, thick, black]          (s0) -- (sR)  node[lbl, midway, above right] {$R,\ r{=}{+}3$};
          \draw[->, thick, green!55!black] (sL) -- (tL1) node[lbl, midway, left]        {$a_1,\ r{=}{+}100$};
          \draw[->, thick, red!70!black]   (sL) -- (tL2) node[lbl, midway, right]       {$a_2,\ r{=}{-}80$};
          \draw[->, thick, black]          (sR) -- (tR)  node[lbl, midway, right]       {$a_0,\ r{=}{+}1$};
        \end{tikzpicture}%
        }}
        \caption{Risky Corridor MDP.}
        \label{fig:risky_mdp}
    \end{subfigure}\hfill
    \begin{subfigure}[b]{0.33\textwidth}
        \centering
        \parbox[c][\panelht][c]{\linewidth}{\centering
        \includegraphics[width=\linewidth]{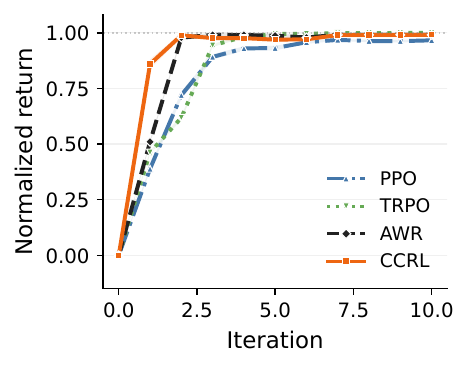}}
        \caption{Normalized expected return.}
        \label{fig:comparison_return}
    \end{subfigure}\hfill
    \begin{subfigure}[b]{0.33\textwidth}
        \centering
        \includegraphics[width=\linewidth]{results/policy_convergence_trajectory.pdf}
        \caption{Policy-space trajectory.}
        \label{fig:comparison_trajectory}
    \end{subfigure}
    \caption{\textbf{Risky Corridor MDP and convergence
    behavior across PPO, TRPO, AWR, and CCRL.}
    A ``pay-now-gain-later'' MDP designed to test how
    locally greedy behavior by CPI surrogate methods can lead to
    slow convergence: the rewarding path at $s_0$
    requires paying an immediate loss in exchange for a large
    downstream payoff, while the safe alternative pays a small
    immediate reward and yields a low total return.  All methods use the same per-iteration
    interaction budget ($32$ episodes per update, $10$ updates) and
    are individually hyperparameter-tuned
    (protocol in Appendix~\ref{app:risky-corridor-table}).
    \emph{(a)} The MDP;
    \emph{(b)} expected return over policy iterations, normalized so
    that $1$ is optimal and $0$ is the uniform-policy return (mean
    over $100$ seeds; shaded band is $\pm 1$ SE): CCRL is the only
    method to reach
    $90\%$ of optimal in under two iterations and also achieves
    the highest area under the learning curve (quantitative
    numbers in Appendix~\ref{app:risky-corridor-table});
    \emph{(c)} policy trajectory in
    $(q, u) = (\pi(a_L|s_0),\, \pi(a_1|s_L))$ space, from uniform
    $(0.5, 0.5)$ toward the optimum $(1, 1)$, over the normalized-return
    landscape (shading red-to-green for low-to-high return, with
    labelled contours): all methods identify the rewarding action at
    $s_L$, only CCRL commits to the pay-now direction at
    $s_0$.\vspace{-.5em}}
    \label{fig:comparison}
\end{figure}

The Risky Corridor (Figure~\ref{fig:comparison}(a)) places a pay-now/gain-later choice at $s_0$: the optimal action $L$ pays an immediate loss for a large downstream payoff, while the safe action $R$ yields a small reward into a low-return basin.
Under the uniform policy this inverts the local signal---%
$A^{\pi_b}(s_0, L) = -7$ (Appendix~\ref{app:risky-corridor-advantage})---so a $k = 1$ update reads the negative advantage as evidence against the optimal action, whereas $k = 2$ couples $s_0$ with $s_L$ and can commit to $L$ once the downstream action is identified.  We compare PPO, TRPO, AWR, and CCRL ($k = 2$) over $100$ seeds under a common budget of $32$ episodes per update over $10$ updates, with per-method tuning (Appendix~\ref{app:risky-corridor-table}).

All four methods come close to the optimum (final normalized return $0.966$--$0.9996$; Appendix~\ref{app:risky-corridor-table}) but differ in speed: CCRL reaches $90\%$ of optimal in $1.38$ iterations on average---the only method under two---against $3.29$, $3.13$, and $2.09$ for PPO, TRPO, and AWR.  Figure~\ref{fig:comparison}(c) shows the mechanism: all four identify the downstream action at $s_L$ within roughly one iteration, but only CCRL simultaneously commits to $L$ at $s_0$; the $k = 1$ methods first move away from $L$ in response to its locally negative advantage.  This supports~(i)---the $k = 2$ surrogate recovers credit the $k = 1$ methods miss---and the pathology sharpens as the reward chain lengthens, which we study next.

\vspace{-1em}
\subsection{Gains amplify with dependency horizon}
\label{sec:risky-chain}
\vspace{-.5em}

The Risky Chain (Figure~\ref{fig:risky-chain}) generalizes the corridor to $N$ stages: the large terminal reward is reachable only by continuing through all $N-1$ stages and then taking the correct terminal action (optimal return $100 - 5(N-1)$).  The corridor of Section~\ref{sec:risky-corridor} is the $N = 2$ case; with $\gamma = 1$ we vary $N \in \{4, 5, 6\}$ to test whether the multi-step advantage seen there grows with the dependency horizon, as that analysis predicts.  We benchmark CCRL against PPO (the most widely used iterative on-policy method), tuning both identically: (i) an equal budget---$128$ trajectories per update and $200$ iterations, each (algorithm, $N$, $k$) cell tuned from scratch; (ii) the entropy bonus swept over the same grid for both, so PPO is never starved of exploration; and (iii) an exhaustive grid search---we rank every configuration on $10$ tuning seeds by nAUC, keep the top $12$, and report the best of those on $300$ validation seeds disjoint from the tuning seeds (full grids in Appendix~\ref{app:chain-hparams}).
\vspace{-.5em}
\begin{table}[h]
\caption{\textbf{PPO vs.\ CCRL on the Risky Chain.} For each depth $N$ we report the final
normalized return (F; $\uparrow$, $1$=optimal, $0$=uniform-policy return),
nAUC ($\uparrow$; mean normalized return over the $200$ training iterations),
and $t_{90}$ ($\downarrow$; iterations to reach $90\%$ of optimal), each as
mean over $300$ validation seeds with the stderr in parentheses in
units of the last displayed digit (e.g.\ $0.939(13)=0.939\pm0.013$,
$9.4(23)=9.4\pm2.3$).  Best per metric column (within each $N$) in bold. Full-depth CCRL reaches $90\%$ of optimal
within about one iteration at every $N$ and ends at a normalized return of
$1.000$, whereas PPO's final return and $t_{90}$ steadily degrade as the chain
deepens.}
\label{tab:chain-mlp}
\centering
\scriptsize
\setlength{\tabcolsep}{4pt}
\renewcommand{\arraystretch}{1.0}
\begin{tabular}{@{}l ccc ccc ccc@{}}
\toprule
& \multicolumn{3}{c}{$N=4$} & \multicolumn{3}{c}{$N=5$} & \multicolumn{3}{c}{$N=6$} \\
\cmidrule(lr){2-4}\cmidrule(lr){5-7}\cmidrule(lr){8-10}
Method & F$\uparrow$ & nAUC$\uparrow$ & $t_{90}\!\downarrow$ & F$\uparrow$ & nAUC$\uparrow$ & $t_{90}\!\downarrow$ & F$\uparrow$ & nAUC$\uparrow$ & $t_{90}\!\downarrow$ \\
\midrule
PPO          & 0.939(13) & 0.937(12) & 3.9(13) & 0.920(15) & 0.920(14) & 5.4(16) & 0.916(15) & 0.916(14) & 9.4(23) \\
CCRL $k{=}1$ & \textbf{1.000(0)} & \textbf{0.997(0)} & 1.3(0) & 0.980(8) & 0.975(8) & 5.7(16) & \textbf{1.000(0)} & 0.995(0) & 1.6(0) \\
CCRL $k{=}2$ & 0.990(6) & 0.987(6) & 3.2(11) & 0.960(11) & 0.957(11) & 9.2(23) & \textbf{1.000(0)} & \textbf{0.996(0)} & 1.4(0) \\
CCRL $k{=}3$ & 0.990(5) & 0.987(5) & 3.2(12) & 0.997(3) & 0.991(3) & 2.3(7) & \textbf{1.000(0)} & \textbf{0.996(0)} & 1.4(0) \\
CCRL $k{=}4$ & \textbf{1.000(0)} & \textbf{0.997(1)} & \textbf{1.2(1)} & 0.996(4) & 0.989(5) & 2.7(9) & 0.996(4) & 0.993(4) & 2.0(7) \\
CCRL $k{=}5$ & $-$ & $-$ & $-$ & \textbf{1.000(0)} & \textbf{0.996(1)} & \textbf{1.2(1)} & 0.996(4) & 0.993(4) & 2.0(7) \\
CCRL $k{=}6$ & $-$ & $-$ & $-$ & $-$ & $-$ & $-$ & \textbf{1.000(0)} & \textbf{0.996(0)} & \textbf{1.3(0)} \\
\bottomrule
\end{tabular}
\end{table}

\textbf{Results.} We measure the final normalized return ($1$=optimal, $0$=uniform), the normalized area under the learning curve (nAUC, the mean normalized return over the $200$ iterations), and $t_{90}$, the iterations needed to reach $90\%$ of optimal.  Full-depth CCRL ($k{=}N$) reaches $90\%$ of optimal within about one iteration at every depth ($t_{90}\le 1.3$ for $N\ge 4$) and converges to a normalized return of $1.000$.  Tuned PPO, by contrast, degrades steadily as the chain lengthens: its final return falls from $0.939$ at $N{=}4$ to $0.920$ and $0.916$ at $N{=}5,6$, and its $t_{90}$ climbs from $3.9$ to $9.4$ iterations. Since the two methods share the same networks and budget, this widening gap points to the policy-improvement operator rather than to capacity or budget, and it is sharpest in $t_{90}$---flat in $N$ for CCRL but growing for PPO, the $O(1)$-vs-$O(N)$ signature of one-step bootstrapping whose error compounds with depth.  Among CCRL depths, $k{=}N$ is best or tied on every metric at every $N$ and uniquely fastest to the $90\%$ threshold, while intermediate $k$ is occasionally non-monotonic (per-arm tuning variance).  Together these support question~(ii): the multi-step advantage grows with the dependency horizon.

\vspace{-1em}
\subsection{CCRL is competitive on classic control}
\label{sec:classic-control}
\vspace{-1em}
\begin{figure}[h]
    \centering
    \includegraphics[width=\linewidth]{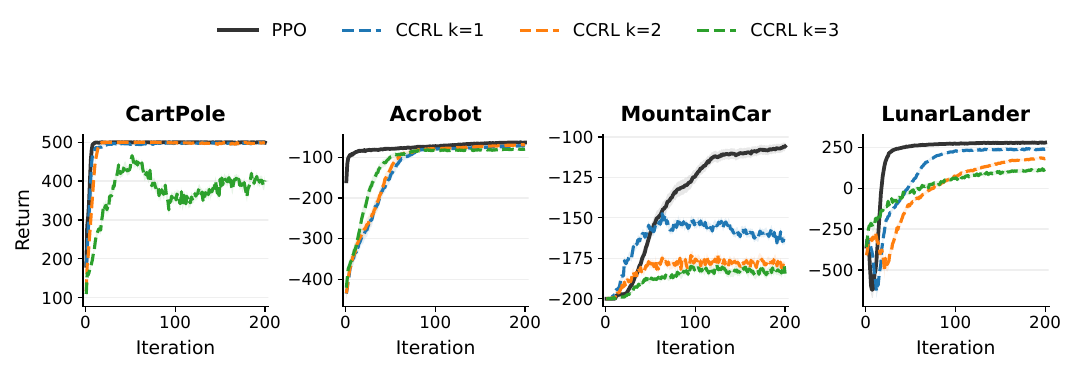}
    \caption{\textbf{Classic control: CCRL vs.\ PPO over $100$ seeds.}
    Return over $200$ iterations (mean over $100$ seeds; shaded band is
    $\pm 1$ SE) for PPO and CCRL at
    $k \in \{1, 2, 3\}$ on CartPole, Acrobot, MountainCar, and
    LunarLander.  CCRL ($k = 1$) is competitive with PPO across all
    tasks. Wider windows do not consistently improve performance here;
    the clearest declines occur on CartPole and LunarLander
    (Section~\ref{sec:classic-control}).}\vspace{-1em}
    \label{fig:control}
\end{figure}

This experiment tests question~(iii): whether CCRL holds up against a strong PPO baseline on standard control. On four discrete-action benchmarks (CartPole, Acrobot, MountainCar, LunarLander) we compare tuned PPO against CCRL at $k \in \{1,2,3\}$, each tuned per configuration and run over $100$ seeds for $200$ iterations (Figure~\ref{fig:control}; setup in Appendix~\ref{app:classic-control}). At $k=1$, CCRL is competitive with PPO, matching it on CartPole and Acrobot and trailing modestly on MountainCar and LunarLander. Increasing $k$ offers no consistent improvement, with the clearest declines on CartPole and LunarLander.

\noindent\textbf{Choosing the window length.} On these classic-control tasks, dense per-step rewards already provide useful feedback at $k=1$, while longer windows make the update more sensitive to advantage-estimation error and importance-weight variance. The CartPole diagnostics show both effects (Appendices~\ref{app:adv-quality} and~\ref{app:pdis-variance}), consistent with the observed results. Together with the delayed-reward experiments, these findings suggest choosing $k$ according to how many decisions must be credited jointly.
\vspace{-1.2em}
\subsection{CCRL can outperform PPO on a stochastic, medium-horizon task}
\label{sec:sepsis}
\vspace{-.75em}

Sepsis treatment is a sequential-decision problem in clinical RL \citep{komorowski2018ai, oberst2019counterfactual}. For beta-lactam antibiotics, therapeutic exposure depends on the time that free drug concentration exceeds a target, and dosing in sepsis can affect whether that target is reached \citep{guilhaumou2019betalactam,delattre2020loading}. Motivated by this exposure dependence, our \emph{Therapeutic-Threshold Sepsis MDP} provides a stylized stochastic task with delayed treatment effects.  A state pairs illness severity $h$ with an accumulated dose $d$, and at each of $H{=}20$ steps the clinician chooses to \emph{rest}, \emph{dose} (advance $d$ toward the therapeutic threshold, with no immediate benefit), or \emph{administer} (an improvement in severity with high probability once the threshold $d{=}D{-}1$ is reached, but ineffective and dose-resetting otherwise).  Survival, a sparse terminal reward, requires loading the dose state before treatment can improve severity, so the actions that enable recovery precede its payoff. With $M{=}5$ severity levels and $D{=}3$ dose levels, the optimal policy survives $84.0\%$ of episodes (full dynamics in Appendix~\ref{app:sepsis}).

In this regime, each \emph{dose} action changes the effectiveness of a later administration without producing an immediate recovery reward. The moderate horizon keeps this delayed structure learnable, while a tight $y$-clip controls transition-induced variance.  Since data is the binding constraint in clinical RL, we compare how quickly each method reaches a good policy, sweeping each method's variance knob (CCRL's $y_{\max}$, PPO's $\varepsilon$) with the remaining hyperparameters sub-tuned (parametric MLP; $500$ updates; $100$ seeds; Table~\ref{tab:sepsis}). Trained to convergence both reach the same near-optimal survival (CCRL $0.828{\pm}0.009$, PPO $0.822{\pm}0.009$; $\approx\!98\%$ of the $0.840$ optimum), so the task does not separate them at the endpoint---it separates them on \emph{speed}: CCRL's mean survival over training is $0.815{\pm}0.008$ versus PPO's $0.732{\pm}0.009$, i.e.\ CCRL is near-optimal almost from the start while PPO climbs slowly.  Both levers are necessary: at a loose clip the $k{>}1$ objective is too noisy to help (Table~\ref{tab:sepsis-yk}: deeper truncation collapses to $0.11$ at $k{=}3$, $y_{\max}{=}5$), while tightening $y_{\max}$ makes every depth usable and lets multi-step credit exploit delayed dosing effects that PPO learns more slowly.

\begin{table}[t]
\centering
\scriptsize
\begin{subtable}[t]{0.60\textwidth}
\centering
\setlength{\tabcolsep}{3pt}
\renewcommand{\arraystretch}{1.05}
\begin{tabular}{@{}c cc @{\quad} c cc@{}}
\toprule
\multicolumn{3}{c}{\emph{PPO} (clip $\varepsilon$)} & \multicolumn{3}{c}{\emph{CCRL} $k{=}2$ (clip $y_{\max}$)} \\
\cmidrule(lr){1-3}\cmidrule(lr){4-6}
$\varepsilon$ & F$\uparrow$ & nAUC$\uparrow$ & $y_{\max}$ & F$\uparrow$ & nAUC$\uparrow$ \\
\midrule
$0.05$ & \textbf{0.822(9)} & \textbf{0.732(9)} & $1.0$ & \textbf{0.828(9)} & \textbf{0.815(8)} \\
$0.10$ & 0.771(21) & 0.709(11) & $1.5$ & 0.796(19) & 0.786(18) \\
$0.20$ & 0.765(23) & 0.693(16) & $2.0$ & 0.778(22) & 0.762(22) \\
$0.30$ & 0.754(24) & 0.689(16) & $3.0$ & 0.719(29) & 0.684(30) \\
$0.50$ & 0.736(28) & 0.652(22) & $5.0$ & 0.750(25) & 0.681(25) \\
\bottomrule
\end{tabular}
\caption{Each method's variance knob swept.}
\label{tab:sepsis}
\end{subtable}\hfill
\begin{subtable}[t]{0.38\textwidth}
\centering
\setlength{\tabcolsep}{4pt}
\renewcommand{\arraystretch}{1.02}
\begin{tabular}{@{}l ccc@{}}
\toprule
 & \multicolumn{3}{c}{final survival $\uparrow$} \\
\cmidrule(lr){2-4}
$k \,\backslash\, y_{\max}$ & $5$ & $3$ & $2$ \\
\midrule
$2$ & 0.721(28) & 0.708(30) & \textbf{0.760(24)} \\
$3$ & 0.110(29) & 0.504(41) & \textbf{0.603(38)} \\
$4$ & 0.251(39) & 0.434(42) & \textbf{0.552(40)} \\
\bottomrule
\end{tabular}
\caption{CCRL vs.\ depth $k$ and clip $y_{\max}$.}
\label{tab:sepsis-yk}
\end{subtable}
\caption{\textbf{Therapeutic-Threshold Sepsis} (parametric MLP; $100$ seeds;
$500$ updates; DP-optimal survival $0.840$).  Values are mean over seeds with
the standard error in parentheses in units of the last digit (e.g.\
$0.822(9){=}0.822{\pm}0.009$).  \emph{(a)} Final survival and nAUC (mean
survival over training): with each variance knob tuned, PPO and CCRL converge
to the \emph{same} final survival ($\approx0.82$), but CCRL's nAUC is far
higher---it reaches near-optimal survival much sooner.  Best per column in
bold.  \emph{(b)} Controlled sweep isolating the mechanism: tightening
$y_{\max}$ makes multi-step credit usable at every depth (row-max in bold),
while a loose clip collapses deep $k$.}\vspace{-1.25em}
\label{tab:sepsis-both}
\end{table}

\vspace{-1em}
\section{Related Work}
\label{sec:related}
\vspace{-.75em}
CCRL connects to three lines of work, each discussed in the paper where it is used: the CPI-surrogate family it generalizes (CPI, NPG, TRPO, AWR; Section~\ref{sec:special_cases}); multi-step credit-assignment methods that shape the value or advantage \emph{target} rather than the objective (TD($\lambda$), GAE, Retrace); and the difference-of-convex and policy-gradient-convergence literatures our analysis builds on (Section~\ref{sec:convergence}).  Closest in spirit is POIS \citep{metelli2018policy}, which likewise treats importance sampling as an \emph{optimization} objective; CCRL differs by exposing its difference-of-convex structure and using the truncation depth $k$ to couple consecutive decisions.  A comprehensive, fully-referenced treatment is given in Appendix~\ref{app:related}.

\vspace{-.75em}
\section{Conclusion}
\label{sec:conclusion}
\vspace{-.5em}
This work characterized iterative on-policy improvement as a difference-of-convex (DC) program in log-density-ratio coordinates and proposed \emph{Convex-Concave RL} (CCRL), which recovers the CPI objective and the constrained programs underlying TRPO and AWR, and extends the formulation along the divergence order $\alpha$ and the truncation depth $k$.  Theoretically, we established, under the stated conditions, an $O(1/R)$ bound on the smallest squared inner-loop step for any $k$, with d-stationary limit points under LICQ, and, at the exact-improvement setting $k{=}H$ with a $\chi^2$ trust region, a conditional geometric contraction of the performance gap toward a neighborhood under uniform support, bounded concentrability, and a small global inner-solve gap.  Empirically, on the risky corridor the multi-step objective recovered credit that convex-surrogate baselines missed; on deeper chains, its advantage over PPO widened.  On classic control it is competitive with PPO but not consistently better, consistent with noisy advantage estimates and $k$-step PDIS variance; on a stochastic, horizon-$20$ sepsis simulator it shows a convergence edge.

\noindent\textbf{Limitations and open directions.}~Each SCP round requires a convex solve, and larger $k$ increases the sensitivity of the update to advantage-estimation error and PDIS variance. The convergence guarantees concern the theoretical tabular algorithm; the practical implementation uses estimated objectives, final normalization, and parametric policy fitting. The outer bound additionally requires uniform support, bounded concentrability, and a small gap to the global inner optimum. Future work could reduce multi-step variance, select $k$ adaptively, scale the solver to larger domains, and establish convergence of the full inner-loop sequence.

\section*{Contributions}

\textbf{Authors and their abbreviations:}
Shripad Deshmukh: SD, 
Yaswanth Chittepu: YC,
Dhawal Gupta: DG,
Philip Thomas: PT,
Scott Niekum: SN.

\textbf{Individual author contributions:}

\textbf{SD}: Through ideation with YC and DG, SD formalized multi-step trust-region policy improvement as a difference-of-convex (DC) program. PT and SN advised on different aspects of this formalization. With YC's help, SD designed a sequential convex programming (SCP)-based algorithm for convex-concave RL. He also designed failure-case MDPs for current surrogate-objective methods, made theoretical connections to existing policy gradient methods, scaled up the experiments, diagnosed the current limitations of SCP-based CCRL, and identified a realistic healthcare use case demonstrating its effectiveness over PPO. With the help of AI tools, SD developed the convergence theory for the framework and contrasted it with existing convergence analyses of policy gradient methods in RL. Finally, he wrote the manuscript under the guidance of PT and SN, with suggestions from YC and DG.

\textbf{YC}: YC helped SD develop the DC programming formalism, which originated from earlier unsuccessful attempts to frame the policy improvement objective as a geometric program.

\textbf{DG}: DG contributed to the early ideation underlying CCRL, drawing on prior work on trust-region policy improvement with batch reuse developed with collaborators. In particular, DG brought the importance-weighted policy improvement perspective and contributed the idea of using log importance ratios, $y(s,a)$, as the optimization variables, motivated by the useful convex structure this parameterization exposes.

\textbf{PT}: PT advised on the development of the formalization, methods, and experiments, and provided feedback on early drafts.

\textbf{SN}: SN advised on the development of the formalization, methods, experiments, and writing.

\bibliographystyle{rlj}
\bibliography{main}

\beginSupplementaryMaterials

\appendix

\noindent\textbf{Appendix guide.}
\begin{itemize}[leftmargin=*,nosep]
\item Appendix~\ref{app:fdiv_convex}: trust-region forms, divergence convexity, and KL endpoints.
\item Appendix~\ref{app:proofs_dc}: DC reformulation and single-step correspondences.
\item Appendix~\ref{app:lagrangian_dc}: SCP subproblem, KKT conditions, and solver details.
\item Appendix~\ref{app:algorithm}: theoretical and practical CCRL algorithms.
\item Appendix~\ref{app:proofs_convergence}: directional stationarity and convergence proofs.
\item Appendix~\ref{app:experiments}: additional results and experimental setups.
\item Appendix~\ref{app:related}: extended related work.
\end{itemize}


\section{Convexity of $\alpha$-Divergence Constraints}
\label{app:fdiv_convex}

In policy space, a convex $f$-divergence gives a convex trust-region constraint (Proposition~\ref{prop:fdiv_convex}). After the change to log-density ratios $y$, the finite $\alpha$-divergence forms remain convex for $\alpha<0$ or $\alpha>1$ (Proposition~\ref{prop:trust}). The KL endpoints need separate treatment because an SCP subproblem can leave the validity manifold. Propositions~\ref{prop:reverse_kl} and~\ref{prop:forward_kl_nonconvex} show how their expressions behave away from that manifold. Table~\ref{tab:trust_region_variants} lists the three trust-region forms used in the paper.

\begin{table}[H]
\caption{Three variants of the trust-region constraint
(Section~\ref{sec:trust_logratio}), each in its general $f$-divergence
form (policy space) and its corresponding $\alpha$-divergence form in
log-density-ratio coordinates.}
\label{tab:trust_region_variants}
\centering
\resizebox{\linewidth}{!}{%
\renewcommand{\arraystretch}{1.5}
\begin{tabular}{@{}lll@{}}
\toprule
Variant & $f$-divergence (policy space) & $\alpha$-divergence ($y$-space) \\
\midrule
Trajectory-level &
  $\E_{\tau \sim p_{\pi_n}}\!\bigl[f(p_\pi(\tau)/p_{\pi_n}(\tau))\bigr] \le \delta$ &
  $\frac{1}{\alpha(\alpha-1)} \E_{\tau \sim p_{\pi_n}}\!\bigl[e^{\alpha \sum_t y(S_t, A_t)} - 1\bigr] \le \delta$ \\
Strict per-state &
  $\E_{A \sim \pi_n(\cdot|s)}\!\bigl[f(\pi(A|s)/\pi_n(A|s))\bigr] \le \delta_s,\; \forall s$ &
  $\frac{1}{\alpha(\alpha-1)} \bigl(\E_{A \sim \pi_n(\cdot|s)}\!\bigl[e^{\alpha\, y(s,A)}\bigr] - 1\bigr) \le \delta_s,\; \forall s$ \\
Expected per-state &
  $\E_{S \sim d^{\pi_n}}\!\bigl[\E_{A \sim \pi_n(\cdot|S)}\!\bigl[f(\pi(A|S)/\pi_n(A|S))\bigr]\bigr] \le \delta$ &
  $\E_{S \sim d^{\pi_n}}\!\Bigl[\frac{1}{\alpha(\alpha-1)}\bigl(\E_{A \sim \pi_n(\cdot|S)}\!\bigl[e^{\alpha\, y(S,A)}\bigr] - 1\bigr)\Bigr] \le \delta$ \\
\bottomrule
\end{tabular}%
}
\end{table}

\subsection{$f$-divergence convexity in $\pi$}

\begin{proposition}[Convexity of $f$-divergence in $\pi$]
  \label{prop:fdiv_convex}
Let $f : \R_+ \to \R$ be convex with $f(1) = 0$, and let $\pi_n$ have full support.  Then $D_f(\pi(\cdot|s) \| \pi_n(\cdot|s)) = \sum_a \pi_n(a|s)\, f(\pi(a|s)/\pi_n(a|s))$ is convex in $\pi(\cdot|s)$ for every $s$.
\end{proposition}

\begin{proof}
Fix $s$ and write $q_a=\pi_n(a|s)>0$. Take two policies $\pi^0,\pi^1$ and $\lambda\in[0,1]$, and set $\pi^\lambda=\lambda\pi^0+(1-\lambda)\pi^1$. For each action $a$, convexity of $f$ gives
\begin{align*}
q_a f\!\left(\frac{\pi^\lambda(a|s)}{q_a}\right)
&=q_a f\!\left(\lambda\frac{\pi^0(a|s)}{q_a}+(1-\lambda)\frac{\pi^1(a|s)}{q_a}\right) \\
&\leq\lambda q_a f\!\left(\frac{\pi^0(a|s)}{q_a}\right)
 +(1-\lambda)q_a f\!\left(\frac{\pi^1(a|s)}{q_a}\right).
\end{align*}
Summing over $a$ yields $D_f(\pi^\lambda(\cdot|s)\|q)\leq\lambda D_f(\pi^0(\cdot|s)\|q)+(1-\lambda)D_f(\pi^1(\cdot|s)\|q)$, which is convexity in $\pi(\cdot|s)$.
\end{proof}

Thus the policy-space trust region is convex; for $k>1$, the non-convexity of the policy-space per-iteration program comes from the objective $J_k$.

\subsection{$\alpha$-divergence in $y$-space: finite $\alpha$}

\noindent\textbf{Proposition~\ref{prop:trust} (restated).} For every finite $\alpha<0$ or $\alpha>1$, the trajectory-level, strict per-state, and expected per-state $\alpha$-divergence constraints are convex in $y$ on all of $\R^{|\cS||\cA|}$.

\begin{proof}[Proof of Proposition~\ref{prop:trust}]
In Table~\ref{tab:trust_region_variants}, the exponent is $\alpha\ell(y)$, where $\ell(y)$ is either the sum of $y(S_t,A_t)$ along a trajectory or one coordinate $y(s,a)$. Both are affine in $y$. For any $y,z$ and $\lambda\in[0,1]$, convexity of the exponential gives
\[
e^{\alpha\ell(\lambda y+(1-\lambda)z)}
=e^{\lambda\alpha\ell(y)+(1-\lambda)\alpha\ell(z)}
\leq\lambda e^{\alpha\ell(y)}+(1-\lambda)e^{\alpha\ell(z)}.
\]
The exponential term is convex for any $\alpha$; the restriction comes from its prefactor. When $\alpha<0$ or $\alpha>1$, $1/[\alpha(\alpha-1)]>0$, so the non-negative sums and fixed-policy expectations in Table~\ref{tab:trust_region_variants} remain convex. Subtracting the constant does not change this. Hence each $D_\alpha(y)$ is convex and $D_\alpha(y)\leq\delta$ is a convex constraint. For $0<\alpha<1$, the prefactor is negative, so these expressions are concave in $y$ and their sublevel sets need not be convex. The endpoints $\alpha=0,1$ are treated below.
\end{proof}

Hence all three forms in Table~\ref{tab:trust_region_variants} give convex constraints in $y$ for these finite values of $\alpha$.

\subsection{The KL endpoints: one convex, one not}

\begin{proposition}[Reverse-KL endpoint on the validity manifold]
  \label{prop:reverse_kl}
The off-manifold limit of the strict per-state formula in Table~\ref{tab:trust_region_variants} as $\alpha\to0^-$ is $-\sum_a\pi_n(a|s)y(s,a)$, a linear function of $y$. On $f_s(y)=1$ it equals $\mathrm{KL}(\pi_n(\cdot|s)\|\pi(\cdot|s))$. For the normalized policy defined at an arbitrary $y$, the reverse KL instead equals $-\sum_a\pi_n(a|s)y(s,a)+\log f_s(y)$; this is a different, convex off-manifold extension with the same values on $f_s=1$.
\end{proposition}

\begin{proof}
Fix $s$ and write $q_a=\pi_n(a|s)$, $y_a=y(s,a)$, and $F(\alpha)=\sum_a q_a e^{\alpha y_a}$. Since $\sum_a q_a=1$, we have $F(0)=1$ and $F'(0)=\sum_a q_a y_a$. Therefore
\[
\lim_{\alpha\to0^-}\frac{F(\alpha)-1}{\alpha(\alpha-1)}
=\lim_{\alpha\to0^-}\frac{F(\alpha)-F(0)}{\alpha}\frac{1}{\alpha-1}
=-\sum_a q_a y_a.
\]
If $f_s(y)=1$, then $\pi(a|s)=q_a e^{y_a}$ and $-y_a=\log[q_a/\pi(a|s)]$. The limit is therefore $\sum_a q_a\log[q_a/\pi(a|s)]=\mathrm{KL}(q\|\pi)$.

At an arbitrary $y$, the normalized policy is $\pi_y(a|s)=q_a e^{y_a}/f_s(y)$. Now $\log[q_a/\pi_y(a|s)]=-y_a+\log f_s(y)$, so $\mathrm{KL}(q\|\pi_y)=-\sum_a q_a y_a+\log f_s(y)$. To check convexity, put $w_a=\pi_y(a|s)$. Differentiating $\log f_s$ gives $\partial\log f_s/\partial y_a=w_a$ and, for any vector $v$,
\[
v^\top\nabla^2\log f_s(y)\,v
=\sum_a w_a v_a^2-\Bigl(\sum_a w_a v_a\Bigr)^2\geq0,
\]
since $w_a\geq0$, $\sum_a w_a=1$, and Jensen's inequality gives $(\sum_a w_a v_a)^2\leq\sum_a w_a v_a^2$. Thus $\log f_s(y)$ is convex, and adding it to the linear term gives the claimed convex extension. When $f_s(y)=1$, the added term vanishes.
\end{proof}

An intermediate SCP point need not satisfy $f_s(y)=1$, so the linear $\alpha\to0^-$ limit and the reverse KL of its normalized policy must be distinguished.

\begin{proposition}[Forward-KL endpoint is not convex in $y$]
  \label{prop:forward_kl_nonconvex}
On the validity manifold, the strict per-state $\alpha$-divergence converges as $\alpha \to 1^+$ to
  \[
    \mathrm{KL}\bigl(\pi(\cdot|s) \,\|\, \pi_n(\cdot|s)\bigr)
    \;=\; \sum_a \pi_n(a|s)\, e^{y(s,a)}\, y(s,a)
    \qquad (\text{on } \mathcal{M}).
  \]
The right-hand expression, viewed as a function on $\R^{|\cA|}$ (i.e., extended off the manifold), is \emph{not} convex.
\end{proposition}

\begin{proof}
Write $F(\alpha)=\sum_a\pi_n(a|s)e^{\alpha y(s,a)}$. On the validity manifold $F(1)=1$, so the limit as $\alpha\to1^+$ is $F'(1)=\sum_a\pi_n(a|s)e^{y(s,a)}y(s,a)=\mathrm{KL}(\pi(\cdot|s)\|\pi_n(\cdot|s))$.

For the off-manifold expression, the second derivative with respect to any coordinate $y(s,a)$ is $\pi_n(a|s)e^{y(s,a)}(y(s,a)+2)$, which is negative when $y(s,a)<-2$. Hence the expression is not convex on $\R^{|\cA|}$. The support floor does not exclude such values whenever $\pi_n(a|s)>e^2p_{\min}$.
\end{proof}

\begin{remark}[Consequences]
  \label{rem:kl_endpoint_consequences}
TRPO's reverse-KL endpoint is a convex constraint in $y$ (Proposition~\ref{prop:reverse_kl}). AWR's forward-KL endpoint is convex in policy space (Proposition~\ref{prop:fdiv_convex}), while its off-manifold expression in $y$ is not (Proposition~\ref{prop:forward_kl_nonconvex}). We therefore derive the AWR solution in policy space (Appendix~\ref{app:proof_prop_awr}). CCRL's SCP analysis uses finite $\alpha>1$, including $\alpha=2$, for which the trust-region constraint is convex in $y$.
\end{remark}

\subsection{Members and limits of the $\alpha$-divergence family}

\begin{table}[H]
\centering
\caption{Selected members and limits of the $\alpha$-divergence family
  $D_\alpha(p \| q) = \tfrac{1}{\alpha(\alpha-1)} (\E_q[(p/q)^\alpha] - 1)$.
  Generators use the scaling of this definition for normalized $p$ and $q$.
  ``Convex in $y$'' refers to the off-manifold trust-region expression.}
\label{tab:alpha_div}
\begin{tabular}{llll}
  \toprule
  $\boldsymbol{\alpha}$ & \textbf{Divergence} & \textbf{Convex in $y$} & \textbf{Generator $f(t)$} \\
  \midrule
  $\alpha = 2$           & $\chi^2/2$                & Yes & $(t-1)^2/2$ \\
  $\alpha = -1$          & Reverse $\chi^2/2$        & Yes & $(1-t)^2/(2t)$ \\
  $\alpha \to 1^+$       & Forward KL                & No (limit on $\mathcal{M}$ only) & $t \log t$ \\
  $\alpha \to 0^-$       & Reverse KL                & Yes (linear limit) & $-\log t$ \\
  $\alpha = 1/2$         & Squared Hellinger (scaled) & No & $2(\sqrt{t} - 1)^2$ \\
  $\alpha > 1$           & $\alpha$-divergence        & Yes & $(t^\alpha - 1)/[\alpha(\alpha-1)]$ \\
  $\alpha < 0$           & $\alpha$-divergence        & Yes & $(t^\alpha - 1)/[\alpha(\alpha-1)]$ \\
  \bottomrule
\end{tabular}
\end{table}

\section{Proofs for the DC Reformulation}
\label{app:proofs_dc}

The first proof gives the objective decomposition used by SCP: its positive- and negative-advantage terms are each convex in $y$. The remaining proofs identify what the program becomes at $k=1$. They establish equality with the CPI objective, the constrained policy-space programs underlying TRPO and AWR, and the local approximation leading to NPG.

\subsection{Proof of Theorem~\ref{thm:dc} (DC structure)}
\label{app:proof_thm_dc}

\noindent\textbf{Theorem~\ref{thm:dc} (restated).} For every $k\geq1$, $J_k(y)=J_k^+(y)-J_k^-(y)$, with both $J_k^+$ and $J_k^-$ convex in $y$; their definitions are in~\eqref{eq:jpm}.

\begin{proof}
For a fixed trajectory $\tau$, the window sum $Y_t^k$ in~\eqref{eq:window_sum} is affine in $y$, so $e^{Y_t^k}$ is convex. In~\eqref{eq:jpm}, the coefficients $\gamma^t(A_t^{\pi_n})_+$ and $\gamma^t(A_t^{\pi_n})_-$ are non-negative and independent of $y$. The expectation is also under the fixed distribution $p_{\pi_n}$. Thus $J_k^+$ and $J_k^-$ are each non-negative sums of convex functions.

The decomposition is exact because $(A_t^{\pi_n})_+-(A_t^{\pi_n})_-=A_t^{\pi_n}$ for every $t$. Subtracting the two expressions in~\eqref{eq:jpm} therefore gives $J_k^+-J_k^-=J_k$.
\end{proof}

In minimizing $\Phi=J_k^--J_k^+$, SCP linearizes $J_k^+$ and keeps $J_k^-$ intact, giving a convex subproblem objective.

\subsection{CPI and TRPO cases of Proposition~\ref{prop:cpi}; NPG approximation}
\label{app:proof_prop_cpi}

\noindent\textbf{Proposition~\ref{prop:cpi} (restated).} At $k=1$, $J_1(y)=Z_H L(\pi;\pi_n)$. With expected per-state reverse KL or forward KL, this gives the policy-space constrained programs underlying TRPO or AWR, respectively. The forward-KL case is proved in Appendix~\ref{app:proof_prop_awr}.

\begin{proof}
At $k=1$, each product in~\eqref{eq:pdis_k} contains only $\rho(S_t,A_t)=\pi(A_t|S_t)/\pi_n(A_t|S_t)$. Conditional on $S_t=s$, averaging over $A_t\sim\pi_n(\cdot|s)$ cancels the denominator:
  \[
    \E_{A_t\sim\pi_n(\cdot|s)}\!\left[\rho(s,A_t)A^{\pi_n}(s,A_t)\right]
    =\sum_a\pi(a|s)A^{\pi_n}(s,a).
  \]
Now sum over the states visited at each time. With $\mu_t^{\pi_n}(s)=\Pr(S_t=s\mid\pi_n)$ and $\eta^{\pi_n}(s)=\sum_t\gamma^t\mu_t^{\pi_n}(s)=Z_Hd^{\pi_n}(s)$, we obtain
  \[
    J_1
    = \sum_s \eta^{\pi_n}(s) \sum_a \pi(a|s) A^{\pi_n}(s, a)
    = Z_H \E_{S \sim d^{\pi_n}}\!\Bigl[\sum_a \pi(a|S) A^{\pi_n}(S, a)\Bigr]
    = Z_H\, L(\pi; \pi_n).
  \]
Because $Z_H>0$, maximizing $J_1$ and maximizing $L$ under the same constraint have the same solutions. On the validity manifold, the expected per-state reverse-KL endpoint is $\E_{S\sim d^{\pi_n}}[\mathrm{KL}(\pi_n(\cdot|S)\|\pi(\cdot|S))]$ by Proposition~\ref{prop:reverse_kl}. Using it as the constraint gives TRPO's policy-space CPI-surrogate program.

NPG uses a local approximation of this program \citep{kakade2002npg, schulman2015trpo}. For a smooth policy $\pi_\theta$ near $\theta_n$, let $g_n=\nabla_\theta L(\pi_\theta;\pi_n)|_{\theta_n}$ and let $F_n$ be the Fisher matrix averaged under $d^{\pi_n}(s)\pi_n(a|s)$. The first-order change in $L$ is $g_n^\top(\theta-\theta_n)$, while the second-order reverse-KL term is $\tfrac12(\theta-\theta_n)^\top F_n(\theta-\theta_n)$. Maximizing the former under a bound on the latter gives a step in the direction $F_n^{-1}g_n$ when $F_n$ is invertible (or its pseudoinverse analogue otherwise). This is a local direction, rather than the solution of the full constrained program.
\end{proof}

\subsection{Forward-KL (AWR) case of Proposition~\ref{prop:cpi}}
\label{app:proof_prop_awr}

\noindent\textbf{Proposition~\ref{prop:cpi}, forward-KL case (restated).} At $k=1$, the expected per-state forward-KL constraint gives the policy-space improvement program underlying AWR.

\begin{proof}[Proof of the forward-KL case]
From the first part of Proposition~\ref{prop:cpi}, $J_1=Z_H L$ with $Z_H>0$. The forward-KL endpoint therefore gives the policy-space problem
\[
  \max_{\pi\in\Pi}\ L(\pi;\pi_n)
  \quad\text{s.t.}\quad
  \E_{S\sim d^{\pi_n}}[\mathrm{KL}(\pi(\cdot|S)\|\pi_n(\cdot|S))]\leq\delta.
\]
This is the single-sampling-policy constrained improvement problem used in the AWR derivation \citep{peng2020awr}. In particular, $\pi\in\Pi$ retains the per-state normalization constraint $\sum_a\pi(a|s)=1$.

To derive its policy target, let $\eta_{\mathrm{awr}}>0$ be the multiplier of an active KL constraint and $\nu_s$ the multiplier of normalization at state $s$. Writing $d_s:=d^{\pi_n}(s)$ and $A_{s,a}:=A^{\pi_n}(s,a)$, the policy-space Lagrangian is
\begin{align*}
  \mathcal L(\pi,\eta_{\mathrm{awr}},\nu)
  ={}& \sum_s d_s\sum_a\pi(a|s)A_{s,a} \\
  &-\eta_{\mathrm{awr}}\sum_s d_s\sum_a\pi(a|s)
       \log\frac{\pi(a|s)}{\pi_n(a|s)}
    +\eta_{\mathrm{awr}}\delta
    +\sum_s\nu_s\left(1-\sum_a\pi(a|s)\right).
\end{align*}
For each state with $d_s>0$, differentiating the Lagrangian with respect to $\pi(a|s)$ gives
\[
d_s\left[A_{s,a}-\eta_{\mathrm{awr}}\left(\log\frac{\pi(a|s)}{\pi_n(a|s)}+1\right)\right]-\nu_s=0.
\]
Thus $\log[\pi(a|s)/\pi_n(a|s)]=A_{s,a}/\eta_{\mathrm{awr}}-1-\nu_s/(\eta_{\mathrm{awr}}d_s)$. The last two terms are constant across actions at state $s$. Enforcing $\sum_a\pi(a|s)=1$ determines that constant and yields
\[
  \pi^*(a|s)=\frac{\pi_n(a|s)\exp(A^{\pi_n}(s,a)/\eta_{\mathrm{awr}})}
    {\sum_b\pi_n(b|s)\exp(A^{\pi_n}(s,b)/\eta_{\mathrm{awr}})}.
\]
At states with $d_s=0$, neither the objective nor the KL constraint depends on $\pi(\cdot|s)$, so we may set $\pi^*(\cdot|s)=\pi_n(\cdot|s)$. This is the normalized exponentiated-advantage target used by AWR; its log-ratio $y^*(s,a)=\log[\pi^*(a|s)/\pi_n(a|s)]$ satisfies $f_s(y^*)=1$. The derivation uses convexity of forward KL in policy space (Proposition~\ref{prop:fdiv_convex}), without requiring a globally convex forward-KL expression in $y$ (Proposition~\ref{prop:forward_kl_nonconvex}).
\end{proof}

\section{Lagrangian Form and KKT Conditions for the SCP Subproblem}
\label{app:lagrangian_dc}

We derive the optimality conditions for Step~1 of the theoretical SCP update, Eqs.~\eqref{eq:scp_sub}--\eqref{eq:Cr}. For the strict per-state trust region used in the convergence analysis, that subproblem is
\begin{align*}
  \hat y^{r+1}
  &= \argmin_{y}\; \tilde\Phi_r(y)+\tfrac{\rho}{2}\|y-y^r\|^2 \\
  \text{subject to}\quad D_\alpha^s(y)&\leq\delta
  &&\forall s,\\
  \nabla f_s(y^r)^\top(y-y^r)&=0
  &&\forall s,\\
  y(s,a)&\geq b_r(s,a)
  &&\forall(s,a).
\end{align*}
Here $\tilde\Phi_r(y)=J_k^-(y)-\nabla J_k^+(y^r)^\top y$ is the linearized objective from~\eqref{eq:scp_sub}, and the two functions in its constraints are
\begin{align*}
  f_s(y)&=\sum_a\pi_n(a|s)e^{y(s,a)} &&\text{(Eq.~\eqref{eq:fs})},\\
  D_\alpha^s(y)&=\frac{\sum_a\pi_n(a|s)e^{\alpha y(s,a)}-1}{\alpha(\alpha-1)}
  &&\text{(Eq.~\eqref{eq:alpha_div}; Table~\ref{tab:trust_region_variants})}.
\end{align*}
The solution $\hat y^{r+1}$ is then normalized by the retraction in Eq.~\eqref{eq:retraction}. The conditions below describe an exact subproblem solution; the numerical iteration is one way to approach it.

\subsection{Lagrangian form}

Let $\xi_s\geq0$ multiply the trust-region constraint at state $s$, $\mu^s\in\R$ multiply its tangent-plane equality, and $\zeta_{s,a}\geq0$ multiply the floor constraint at $(s,a)$. The Lagrangian is
\begin{align}
  \mathcal{L}_r(y, \xi, \mu, \zeta)
  &= \tilde\Phi_r(y) + \tfrac{\rho}{2}\|y - y^r\|^2
     + \sum_s \xi_s \bigl(D_\alpha^s(y) - \delta\bigr) \notag \\
  &\quad
     + \sum_s \mu^s\, \nabla f_s(y^r)^\top (y - y^r)
     + \sum_{s,a} \zeta_{s,a}\, \bigl(b_r(s,a) - y(s,a)\bigr).
  \label{eq:app_lagrangian}
\end{align}

\subsection{KKT conditions}

The subproblem is convex and has a minimizer by Proposition~\ref{prop:feas_compact}. At a minimizer $\hat y$ where a constraint qualification holds, the KKT stationarity condition is
\begin{equation}
  0 = \nabla \tilde\Phi_r(\hat y) + \rho(\hat y - y^r)
      + \sum_s \xi_s\, \nabla D_\alpha^s(\hat y)
      + \sum_s \mu^s\, \nabla f_s(y^r)
      - \sum_{s,a} \zeta_{s,a}\, e_{(s,a)},
  \label{eq:app_kkt}
\end{equation}
where $e_{(s,a)}$ is the coordinate basis vector for $(s,a)$. The remaining KKT conditions are primal feasibility, $\xi_s,\zeta_{s,a}\geq0$, and complementary slackness for each trust-region and floor inequality. The equality multipliers $\mu^s$ have no sign restriction.

\subsection{Gradient expressions}

For implementation:
\begin{align}
  \frac{\partial f_s(y)}{\partial y(\bar s, \bar a)}
  &= \mathbf{1}\{\bar s = s\}\, \pi_n(\bar a | \bar s)\, e^{y(\bar s, \bar a)},
  \\[2pt]
  \frac{\partial D_\alpha^s(y)}{\partial y(\bar s, \bar a)}
  &= \mathbf{1}\{\bar s = s\}\, \frac{1}{\alpha - 1}\,
     \pi_n(\bar a | \bar s)\, e^{\alpha\, y(\bar s, \bar a)},
  \\[2pt]
  \frac{\partial J_k^{\pm}(y)}{\partial y(\bar s, \bar a)}
  &= \E_{\tau \sim p_{\pi_n}}\!\left[
       \sum_{t=0}^{H-1} \gamma^t\,
         e^{\sum_{t'=\max(0,t-k+1)}^{t} y(S_{t'}, A_{t'})}\,
         [A^{\pi_n}(S_t, A_t)]^{\pm} \right. \notag \\
  &\qquad\qquad \left. {} \cdot
         \!\sum_{t'=\max(0,t-k+1)}^{t}\! \mathbf{1}\{(S_{t'}, A_{t'}) = (\bar s, \bar a)\}
     \right].
\end{align}

\subsection{Primal-dual gradient solver}

A projected primal-dual iteration can enforce the floor by projection, so it does not update $\zeta$. For the strict per-state trust region, one such iteration is
\begin{align*}
  y^{m+1}
  &= \Pi_{[b_r, \infty)}\!\bigl(
       y^m - \eta_y\, \nabla_y \mathcal{L}_r(y^m, \xi^m, \mu^m, 0)
     \bigr), \\
  \xi_s^{m+1}
  &= [\xi_s^m + \eta_\xi\, (D_\alpha^s(y^m) - \delta)]_+ \quad \forall s, \\
  \mu^{s, m+1}
  &= \mu^{s, m} + \eta_\mu\, \nabla f_s(y^r)^\top (y^m - y^r)
  \quad \text{(sign-free ascent)},
\end{align*}
where $m$ indexes numerical solver steps within SCP round $r$. The theoretical algorithm assumes an exact subproblem solution and applies the retraction~\eqref{eq:retraction} after \emph{every} SCP round. The practical algorithm in Section~\ref{sec:practical} normalizes only after its final round.

\subsection{Why the equality is linearized as an equality}
\label{app:linearize_equality}

We now prove Proposition~\ref{prop:naive_singleton} on the textbook validity split from Section~\ref{sec:algorithm}.

\noindent\textbf{Proposition~\ref{prop:naive_singleton} (restated).} At a valid anchor $y^r$, imposing both $f_s(y)\leq1$ and the linearized inequality $\nabla f_s(y^r)^\top(y-y^r)\geq0$ at every state leaves only $y=y^r$ feasible.

\begin{proof}[Proof of Proposition~\ref{prop:naive_singleton}]
Fix $s$ and write $z_a := y(s,a) - y^r(s,a)$ and $w_a := \pi_n(a|s)\, e^{y^r(s,a)}$.  Feasibility of the anchor gives $\sum_a w_a = 1$ with $w_a > 0$, so $w$ is a probability vector.  The two constraints read $\sum_a w_a e^{z_a} \leq 1$ and $\sum_a w_a z_a \geq 0$.  By Jensen's inequality for the strictly convex exponential, $\sum_a w_a e^{z_a} \geq \exp\bigl(\sum_a w_a z_a\bigr) \geq e^0 = 1$. Hence both inequalities hold with equality; equality in Jensen forces all $z_a$ equal (strict convexity, $w_a > 0$), and $\sum_a w_a z_a = 0$ then forces $z_a = 0$ for every $a$.  This holds at every state, so $y = y^r$.
\end{proof}

The tangent-plane construction of Section~\ref{sec:algorithm} avoids this degeneracy by dropping the one-sided convex inequality and linearizing the equality as an equality; exact validity is restored by the retraction, whose effect on the objective is second order (Lemma~\ref{lem:retraction}) and absorbed by the proximal term (Theorem~\ref{thm:tier1}).

\section{The full CCRL algorithm}
\label{app:algorithm}

Algorithm~\ref{alg:ccrl} assembles the policy update of Section~\ref{sec:algorithm}: collect data under $\pi_n$, estimate advantages, run $R$ tangent-plane SCP rounds with retraction, then recover $\pi_{n+1}$ from $y^R$. The convergence results analyze this tabular update with exact expectations and advantages and an exact solution of each convex subproblem; the data-collection lines show how the corresponding quantities are supplied in use.

\begin{algorithm}[h]
\caption{CCRL: Convex-Concave Reinforcement Learning}
\label{alg:ccrl}
\begin{algorithmic}[1]
\REQUIRE initial policy $\pi_0$, critic $V_\phi$, episodes per iteration $N$,
  trust radius $\delta$, divergence order $\alpha$, support level $p_{\min}$,
  SCP rounds $R$, truncation depth $k$, proximal weight $\rho$
\FOR{$n = 0, 1, 2, \ldots$}
  \STATE Collect $N$ episodes under $\pi_n$
  \STATE Fit critic $V_\phi$; compute advantage estimates
         $\widehat{A}^{\pi_n}(S_t, A_t)$ (e.g.\ via GAE)
  \STATE Initialize $y^0 \leftarrow 0$, floor
         $b_0 \leftarrow y_{\min} = \log(p_{\min}/\pi_n)$
  \FOR{$r = 0, 1, \ldots, R-1$}
    \STATE Solve tangent-plane subproblem~\eqref{eq:scp_sub} with
           anchor $y^r$ and floor $b_r$ to obtain $\hat y^{r+1}$
    \STATE Retract: $y^{r+1}(s,a) \leftarrow \hat y^{r+1}(s,a)
           - \log f_s(\hat y^{r+1})$ \hfill (per-state renormalization)
    \STATE $\gamma_r \leftarrow \max_s \log f_s(\hat y^{r+1})$;\quad
           $b_{r+1} \leftarrow b_r - \gamma_r \mathbf{1}$
           \hfill (floor schedule~\eqref{eq:floor_schedule})
  \ENDFOR
  \STATE Recover
         $\pi_{n+1}(a|s) \leftarrow \pi_n(a|s)\, e^{y^R(s,a)}$
\ENDFOR
\end{algorithmic}
\end{algorithm}

Algorithm~\ref{alg:ccrl-practical} gives the form used in our experiments (Section~\ref{sec:practical}). The changes highlighted in {\dchg{blue}} are the clipped trust-region subproblem without tangent-plane or floor constraints, a single normalization after the final SCP round, and a maximum-likelihood M-step that fits a parametric policy to the resulting target. The expectations in $J_k^\pm$ and in the M-step are estimated from sampled trajectories and states, respectively. The convergence results above do not cover these changes.

\begin{algorithm}[h]
\caption{CCRL, deployed form (parametric E-step / M-step)}
\label{alg:ccrl-practical}
\begin{algorithmic}[1]
\REQUIRE policy network $\pi_\theta$, critic $V_\phi$; trust radius
  $\delta$, divergence order $\alpha$, truncation depth $k$, SCP rounds
  $R$, proximal weight $\rho$, clip $y_{\max}$, M-step epochs $E$
\FOR{$n = 0, 1, 2, \ldots$}
  \STATE Collect episodes under $\pi_{\theta_n}$; fit $V_\phi$ and
         estimate advantages $\widehat{A}^{\pi_n}$ (e.g.\ GAE)
         \hfill (E-step: build the improved target)
  \STATE $y^0 \leftarrow 0$
  \FOR{$r = 0, \ldots, R-1$}
    \STATE \dchg{$y^{r+1} \leftarrow \argmin_{y}
            \bigl( J_k^-(y) - \nabla J_k^+(y^r)^\top y
            + \tfrac{\rho}{2}\|y - y^r\|^2 \bigr)$
            \ \text{s.t.}\ $D_\alpha(y) \leq \delta,\ \,
            \|y\|_\infty \leq y_{\max}$}
            \hfill \dchg{(plain trust region + $y$-clip; no
            validity/floor)}
  \ENDFOR
  \STATE Target
         $\tilde\pi_{n+1}(a\mid s) \propto
         \pi_{\theta_n}(a\mid s)\, e^{y^R(s,a)}$
         \hfill (row-normalized recovery)
  \STATE \dchg{\textbf{M-step:} fit $\pi_\theta$ to the target for $E$
         epochs by minimizing $-\,\mathbb{E}_{S}\sum_a
         \tilde\pi_{n+1}(a\mid S)\log \pi_\theta(a\mid S)$;
         set $\theta_{n+1} \leftarrow \theta$}
         \hfill (critic and entropy losses as in PPO)
\ENDFOR
\end{algorithmic}
\end{algorithm}

\section{Proofs for the Convergence Analysis}
\label{app:proofs_convergence}

The inner-loop proof shows that retraction preserves the trust region and changes the objective by less than the decrease from the convex subproblem. This keeps every round feasible and yields the descent and limit-point results of Theorem~\ref{thm:inner}. The outer-loop proof compares each full-horizon update with a feasible mixture of the current and a greedy policy, giving the conditional bound of Theorem~\ref{thm:outer}. The inner analysis uses $\alpha>1$ and a strict per-state trust region; the outer analysis specializes to $\alpha=2$.

\subsection{The retraction lemma}

\begin{lemma}[Second-order retraction shift]
  \label{lem:retraction}
Let $y^r \in \mathcal{M}$ and let $\hat y \in T_r$ (i.e., $\nabla f_s(y^r)^\top(\hat y - y^r) = 0$ for all $s$) with $\hat y, y^r$ in a convex compact set $K$ (constructed in Proposition~\ref{prop:feas_compact}) on which each $f_s$ is $L_f$-smooth.  Then for every $s$,
  \begin{equation}
    0 \;\leq\; c_s := \log f_s(\hat y)
    \;\leq\; f_s(\hat y) - 1
    \;\leq\; \tfrac{L_f}{2}\, \|\hat y - y^r\|^2,
    \label{eq:retraction_shift}
  \end{equation}
and the retracted point $y^+ := \hat y - \sum_s c_s\, \mathbf{1}_s$ (where $\mathbf{1}_s$ is the indicator vector of state block $s$) satisfies $f_s(y^+) = 1$ for all $s$ and
  \[
    \|y^+ - \hat y\| \;\leq\; \sqrt{|\cS||\cA|}\; \tfrac{L_f}{2}\,
      \|\hat y - y^r\|^2 .
  \]
\end{lemma}

\begin{proof}
\emph{Lower bound.}  Convexity of $f_s$ and tangency give $f_s(\hat y) \geq f_s(y^r) + \nabla f_s(y^r)^\top(\hat y-y^r)=1$ (the same tangent-plane argument used in Proposition~\ref{prop:naive_singleton}). Hence $c_s\geq0$.

\emph{Upper bound.}  By $L_f$-smoothness and tangency, $f_s(\hat y) \leq f_s(y^r) + \nabla f_s(y^r)^\top(\hat y - y^r) + \tfrac{L_f}{2}\|\hat y - y^r\|^2 = 1 + \tfrac{L_f}{2}\|\hat y - y^r\|^2$; and $\log(1 + u) \leq u$ for $u \geq 0$.

\emph{Exactness of the retraction.}  Subtracting the constant $c_s$ from every coordinate of state block $s$ multiplies each term $\pi_n(a|s) e^{\hat y(s,a)}$ by $e^{-c_s} = 1/f_s(\hat y)$, so $f_s(y^+) = f_s(\hat y)/f_s(\hat y) = 1$.

\emph{Displacement.}  Each of the $|\cA|$ coordinates in state block $s$ shifts by exactly $c_s$, so $\|y^+ - \hat y\|^2 = |\cA| \sum_s c_s^2 \leq |\cA|\, |\cS|\, (\max_s c_s)^2$, giving $\|y^+ - \hat y\| \leq \sqrt{|\cS||\cA|}\, \max_s c_s \leq \sqrt{|\cS||\cA|}\, \tfrac{L_f}{2}\, \|\hat y - y^r\|^2$.
\end{proof}

\subsection{Retraction, one-round descent, and the floor budget}
\label{app:feas_induction_lemmas}

The subproblem~\eqref{eq:Cr} uses the scheduled floor $y \geq b_r$: it starts at $b_0 := y_{\min}$ and, after each round, is relaxed by exactly the retraction shift the round absorbs,
\begin{equation}
  \gamma_r \;:=\; \max_s \log f_s(\hat y^{r+1}) \;\geq\; 0,
  \qquad
  b_{r+1} \;:=\; b_r - \gamma_r \mathbf{1},
  \label{eq:floor_schedule}
\end{equation}
The shift $\gamma_r$ comes from the retraction~\eqref{eq:retraction}. We will use one fixed box for the subproblem and its retraction. Set $y_{\min}(s,a):=\log[p_{\min}/\pi_n(a|s)]$ and $U:=\alpha^{-1}\log[(1+\alpha(\alpha-1)\delta)/p_{\min}]$, and define
\begin{equation}
  K:=\{y:y_{\min}(s,a)-1-U\leq y(s,a)\leq U\quad\forall(s,a)\}.
  \label{eq:box_K}
\end{equation}
Let $L_f$ bound the smoothness of every $f_s$ on $K$, let $L_\Phi$ be a Lipschitz constant of $\Phi$ on $K$, and put $\rho_0:=L_\Phi L_f\sqrt{|\cS||\cA|}$. The floor keeps the subproblem inside $K$; the lemmas below show that each round remains feasible and that the total floor relaxation is bounded.
For the full induction, choose
\begin{equation}
  \rho>\rho_0+L_f\bigl(\textstyle\sup_K\Phi-\inf_K\Phi\bigr).
  \label{eq:rho_choice}
\end{equation}
This choice makes the drift budget below at most one.

\begin{lemma}[Retraction preserves the strict per-state trust region]
  \label{lem:R1}
If $y^+(s,a)=y(s,a)-c_s$ with $c_s\geq0$ for every state, then for $\alpha>1$,
$D_\alpha^s(y^+)\leq D_\alpha^s(y)$ for every $s$. In particular, this applies to the retraction~\eqref{eq:retraction} by Lemma~\ref{lem:retraction}.
\end{lemma}

\begin{proof}
Subtracting $c_s$ multiplies $\sum_a\pi_n(a|s)e^{\alpha y(s,a)}$ by $e^{-\alpha c_s}\leq1$. The prefactor $1/[\alpha(\alpha-1)]$ is positive, so the stated inequality follows. The trajectory-level and expected per-state forms in Table~\ref{tab:trust_region_variants} obey the same inequality after integration against their fixed sampling measures.
\end{proof}

\begin{remark}[Scope of Lemma~\ref{lem:R1}]
  \label{rem:R1_scope}
The lemma is specific to $\alpha > 1$.  For $\alpha < 0$ the prefactor $1/[\alpha(\alpha-1)]$ is still positive but $e^{-\alpha c_s} \geq 1$, so the retraction \emph{increases} the divergence; and for the reverse-KL endpoint, whose per-state form is the linear $-\sum_a \pi_n(a|s)\, y(s,a)$ (Proposition~\ref{prop:reverse_kl}), subtracting $c_s$ \emph{adds} $c_s$ to the value.  The inner-loop guarantee (Theorem~\ref{thm:tier1}) is stated for $\alpha > 1$ under the strict per-state trust region. Other trust forms need additional state or trajectory coverage assumptions for the compactness argument.  This covers $\alpha = 2$ ($\chi^2$), the only setting used in the experiments and the one the outer-loop analysis of Section~\ref{sec:outer_convergence} requires.  We make no inner-loop claim for $\alpha < 0$ or for the reverse-KL endpoint; for those, a trust-region margin dominating the retraction shift would have to be reinstated, together with the bookkeeping the schedule below replaces.
\end{remark}

\begin{lemma}[One-round descent and feasibility]
  \label{lem:one_round}
Fix $\alpha>1$ and the strict per-state trust region. If $y^r\in\mathcal C_r$, $b_r\geq y_{\min}-\mathbf1$, and $\rho>\rho_0$, the round-$r$ subproblem has a unique minimizer $\hat y^{r+1}\in K$. Its retraction lies in $K$ and satisfies $f_s(y^{r+1})=1$, $D_\alpha^s(y^{r+1})\leq\delta$ for every $s$, and $y^{r+1}\geq b_{r+1}$. Moreover, with $\hat z_r:=\hat y^{r+1}-y^r$,
  \begin{equation}
    \Phi(y^{r+1})\leq\Phi(y^r)
      -\frac{\rho-\rho_0}{2}\|\hat z_r\|^2.
    \label{eq:tier1_decrease}
  \end{equation}
\end{lemma}

\begin{proof}
The anchor belongs to $\mathcal C_r$, so this closed convex set is nonempty. The trust region gives $y(s,a)\leq U$ for every feasible $y$: its positive summands satisfy $\pi_n(a|s)e^{\alpha y(s,a)}\leq1+\alpha(\alpha-1)\delta$, and $\pi_n(a|s)\geq p_{\min}$. The floor gives $y\geq b_r\geq y_{\min}-\mathbf1$, hence $\mathcal C_r\subset K$. It is compact, and the positive proximal weight makes the objective strongly convex, so the minimizer exists and is unique.

Convexity of $J_k^+$ gives $\tilde\Phi_r(y)\geq\Phi(y)+c_r$, with equality at $y^r$, where $c_r:=J_k^+(y^r)-\nabla J_k^+(y^r)^\top y^r$. Comparing $\hat y^{r+1}$ with the feasible anchor in~\eqref{eq:scp_sub} therefore yields $\Phi(\hat y^{r+1})\leq\Phi(y^r)-\rho\|\hat z_r\|^2/2$.

Tangency and Lemma~\ref{lem:retraction} give $0\leq c_s:=\log f_s(\hat y^{r+1})\leq L_f\|\hat z_r\|^2/2$. Also $f_s(\hat y^{r+1})\leq e^U$, so $c_s\leq U$. The retraction segment starts at $\hat y^{r+1}\geq y_{\min}-\mathbf1$ and subtracts at most $U$ per coordinate; it therefore lies in $K$. By the displacement bound in Lemma~\ref{lem:retraction} and the Lipschitz bound on $\Phi$,
$\Phi(y^{r+1})\leq\Phi(\hat y^{r+1})+\rho_0\|\hat z_r\|^2/2$. Combining the two bounds proves~\eqref{eq:tier1_decrease}.

Retraction gives $f_s(y^{r+1})=1$ by Lemma~\ref{lem:retraction}. Lemma~\ref{lem:R1} preserves each strict per-state trust constraint. Finally, $\hat y^{r+1}\geq b_r$ and $c_s\leq\gamma_r$, so $y^{r+1}\geq b_r-\gamma_r\mathbf1=b_{r+1}$.
\end{proof}

\begin{lemma}[Drift budget for the floor schedule]
  \label{lem:R2}
For any first $R$ rounds satisfying the hypotheses of Lemma~\ref{lem:one_round},
  \[
    \sum_{r=0}^{R-1} \gamma_r
    \;\leq\; \frac{L_f}{2} \sum_{r=0}^{R-1} \|\hat z_r\|^2
    \;\leq\; \frac{L_f\, \bigl(\Phi(y^0) - \inf_K \Phi\bigr)}
                  {\rho - \rho_0}
    \;=:\; B(\rho),
    \qquad \hat z_r := \hat y^{r+1} - y^r .
  \]
\end{lemma}

\begin{proof}
Lemma~\ref{lem:retraction} gives $\gamma_r\leq L_f\|\hat z_r\|^2/2$. Summing the one-round decrease~\eqref{eq:tier1_decrease} yields $(\rho-\rho_0)\sum_{r<R}\|\hat z_r\|^2/2\leq\Phi(y^0)-\Phi(y^R)\leq\Phi(y^0)-\inf_K\Phi$. Combining these inequalities gives the bound.
\end{proof}

\subsection{Feasibility, compactness, and the well-posed induction}
\label{app:proof_prop_feas_compact}

\begin{proposition}[Feasibility, compactness, and the well-posed
induction]
  \label{prop:feas_compact}
Fix $\alpha>1$ and the strict per-state trust region. Suppose $\pi_n(a|s)\geq p_{\min}$ for every $(s,a)$, with $|\cA|p_{\min}<1$. Use the box $K$ and constants $L_f,L_\Phi,\rho_0$ defined above, and choose $\rho$ as in~\eqref{eq:rho_choice}, so that $B(\rho)\leq1$. Then, with the constraint sets~\eqref{eq:Cr} and floor schedule~\eqref{eq:floor_schedule}:
  \begin{enumerate}[leftmargin=2em,itemsep=1pt,label=(\roman*)]
\item \textup{(Base case.)}  $y^0 = 0 \in \mathcal{C}_0$: indeed $f_s(0) = 1$, $D_\alpha(0) = 0 \leq \delta$, and $0 \geq y_{\min} = b_0$ because $\pi_n(a|s) \geq p_{\min}$.
\item \textup{(Induction step.)}  For every $r$: $\mathcal{C}_r \subseteq K$; the round-$r$ subproblem is feasible, since its anchor satisfies $y^r \in \mathcal{C}_r$; and the retracted iterate satisfies exact validity $f_s(y^{r+1}) = 1$ for all $s$, the trust region $D_\alpha(y^{r+1}) \leq \delta$, and the next floor $y^{r+1} \geq b_{r+1}$---hence $y^{r+1} \in \mathcal{C}_{r+1}$.
\item \textup{(Uniform confinement.)}  All iterates remain in $K$ and satisfy the effective support floor $\pi^r(a|s) \geq p_{\min} e^{-B(\rho)} \geq p_{\min} e^{-1}$.
  \end{enumerate}
The limiting feasible set $\mathcal{F}_\infty
   = \{y : D_\alpha^s(y) \leq \delta\ \forall s,\ f_s(y) = 1\ \forall s,\
y \geq b_\infty\}$ (Remark~\ref{rem:Finf}) is compact.
\end{proposition}

\begin{proof}
\emph{Base case.} At $y^0=0$, we have $f_s(0)=1$, $D_\alpha^s(0)=0\leq\delta$, and $b_0=y_{\min}=\log(p_{\min}/\pi_n)\leq0$. Thus $y^0\in\mathcal C_0$, and $b_0\geq y_{\min}-\mathbf1$.

\emph{Induction.} Suppose $y^r\in\mathcal C_r$ and, for every earlier round $j<r$, Lemma~\ref{lem:one_round} applied and its descent inequality~\eqref{eq:tier1_decrease} held. This is true at $r=0$ by the base case. The floor schedule and Lemma~\ref{lem:R2} for the earlier rounds give $b_r\geq y_{\min}-B(\rho)\mathbf1\geq y_{\min}-\mathbf1$. Lemma~\ref{lem:one_round} now applies at round $r$: it gives a unique minimizer, establishes~\eqref{eq:tier1_decrease} for this round, and shows that $y^{r+1}\in\mathcal C_{r+1}$. Applying Lemma~\ref{lem:R2} through round $r$ gives $b_{r+1}\geq y_{\min}-B(\rho)\mathbf1\geq y_{\min}-\mathbf1$, closing the induction. The trust region bounds every feasible coordinate above by $U$ as shown in Lemma~\ref{lem:one_round}, so $\mathcal C_j\subseteq K$ throughout.

\emph{Confinement and compactness.} The floor bound gives $y^r\geq b_r\geq y_{\min}-B(\rho)\mathbf1$ and therefore $\pi^r(a|s)=\pi_n(a|s)e^{y^r(s,a)}\geq p_{\min}e^{-B(\rho)}\geq p_{\min}e^{-1}$. The floors converge because $\gamma_r\geq0$ and their sum is bounded by Lemma~\ref{lem:R2}. The limiting feasible set is closed, and its trust region and limiting floor place it inside the compact box $K$; hence it is compact.
\end{proof}

\begin{remark}[The feasible set $\mathcal{F}_\infty$ and its floor]
  \label{rem:Finf}
Since $\gamma_r\geq0$ and $\sum_r\gamma_r\leq B(\rho)<\infty$ (Lemma~\ref{lem:R2}), the shifts tend to zero and the floors decrease to a limit
  \[
    b_\infty \;:=\; y_{\min}
      - \Bigl(\textstyle\sum_{r \geq 0} \gamma_r\Bigr)\mathbf{1}
    \;\geq\; y_{\min} - B(\rho)\mathbf{1},
  \]
and we define the fixed limiting feasible set
  \[
    \mathcal{F}_\infty \;:=\;
    \{\, y \;:\; D_\alpha^s(y) \leq \delta\ \forall s,\;\;
      f_s(y) = 1\;\forall s,\;\; y \geq b_\infty \,\}.
  \]
Every iterate therefore satisfies $y^r\geq b_r\geq y_{\min}-B(\rho)\mathbf1$, or $\pi^r(a|s)\geq p_{\min}e^{-B(\rho)}$. For fixed $\pi_n,\alpha,\delta,p_{\min}$, the budget is $O(1/\rho)$. The inner-loop theorem refers to this limiting set; if the floor is inactive at a limit point, the limiting and nominal floors have the same local tangent cone there. The floor is needed for compactness (Remark~\ref{rem:floor_needed}).
\end{remark}

\begin{remark}[Why the floor is needed for $\alpha \geq 1$]
  \label{rem:floor_needed}
Without the floor, compactness fails.  Consider $|\cS| = 1$, $|\cA| = 2$, $\pi_n = (1/2, 1/2)$, and $y_M = (-M, \log(2 - e^{-M}))$.  The validity equality holds exactly, while for $\alpha = 2$, $D_2^s(y_M) \to 1/2$ as $M \to \infty$; so if $\delta \geq 1/2$, the sequence $y_M$ stays feasible while one coordinate escapes to $-\infty$.  Similar behavior occurs for any $\alpha \geq 1$, and in particular throughout the regime $\alpha > 1$ of Remark~\ref{rem:R1_scope}: this is why the floor cannot simply be dropped, and why its relaxation has to be budgeted rather than unbounded (Lemma~\ref{lem:R2}).
\end{remark}

\subsection{Directional stationarity}
\label{app:stationarity}

We define the directional-stationarity notion used in the inner-loop convergence guarantee (Theorem~\ref{thm:inner}).

\paragraph{Definition (Tangent cone).}
Let $\mathcal{F} \subseteq \R^d$ be closed and $y^\star \in \mathcal{F}$.  The \emph{tangent cone} of $\mathcal{F}$ at $y^\star$ is
\[
  T_{\mathcal{F}}(y^\star)
  \;:=\; \bigl\{\, d \in \R^d :
    \exists\, t_k \downarrow 0,\ d_k \to d
    \text{ with } y^\star + t_k d_k \in \mathcal{F}\,\forall k
  \,\bigr\}.
\]
Intuitively, $T_{\mathcal{F}}(y^\star)$ is the set of directions along which $\mathcal{F}$ can be approached from $y^\star$ to first order.

\paragraph{Definition (d-stationarity).}
Let $\Phi : \R^d \to \R$ admit one-sided directional derivatives $\Phi'(y^\star; d) := \lim_{t \downarrow 0} \bigl[\Phi(y^\star + t d) - \Phi(y^\star)\bigr] / t$. A point $y^\star \in \mathcal{F}$ is a \emph{directional (d-)stationary point} of $\Phi$ over $\mathcal{F}$ if
\[
  \Phi'(y^\star; d) \;\geq\; 0
  \quad \text{for every } d \in T_{\mathcal{F}}(y^\star).
\]

Since $\Phi=J_k^- - J_k^+$ is a finite sum of exponentials of affine functions, it is smooth and $\Phi'(y^\star;d)=\nabla\Phi(y^\star)^\top d$. Under the linear independence constraint qualification (LICQ), d-stationarity is equivalent to the KKT conditions for the constrained problem.


\subsection{Inner-loop convergence}
\label{app:proof_thm_tier1}

Theorem~\ref{thm:inner} of the main text has the following detailed form. The limit-point argument follows the usual successive-convex-approximation approach \citep{razaviyayn2013unified}, while accounting for the moving tangent plane and floor.

\begin{theorem}[Inner-loop convergence, strict per-state trust region]
  \label{thm:tier1}
Use the strict per-state trust region with $\alpha>1$ and the box~\eqref{eq:box_K}. Choose $\rho$ as in~\eqref{eq:rho_choice}, and let $\hat z_r:=\hat y^{r+1}-y^r$. Then:
  \begin{enumerate}[leftmargin=2em,itemsep=2pt,label=(\roman*)]
\item \textup{(Monotone descent.)} The one-round decrease~\eqref{eq:tier1_decrease} holds for every $r$.
\item \textup{(Proximal-step bound.)}
      \begin{equation}
        \min_{0 \leq r < R} \|\hat z_r\|^2
        \;\leq\;
        \frac{2\bigl(\Phi(y^0) - \inf_K \Phi\bigr)}
             {(\rho - \rho_0)\, R} .
        \label{eq:tier1_rate}
      \end{equation}
In particular, a round with proximal-step residual $\|\hat z_r\| \leq \varepsilon$ is reached within $O(1/\varepsilon^2)$ SCP rounds, and $\|\hat z_r\| \to 0$, $\|y^{r+1} - y^r\| \to 0$.
\item \textup{(Subsequential d-stationarity.)}  Every limit point $y^\star$ of $\{y^r\}$ lies in $\mathcal{F}_\infty$, and if LICQ holds for the active constraints of $\mathcal{F}_\infty$ at $y^\star$, then $y^\star$ is a d-stationary point of $\Phi$ over $\mathcal{F}_\infty$.
  \end{enumerate}
The threshold $\rho_0$ pays for the second-order retraction shift~\eqref{eq:retraction_shift}.
\end{theorem}

The constants in Theorem~\ref{thm:inner}(ii) can be taken as $c=(\rho-\rho_0)/2$ and any $C>0$ at least $2[\Phi(y^0)-\inf_K\Phi]/(\rho-\rho_0)$; the floor limit in part~(i) is established in Remark~\ref{rem:Finf}.

\begin{proof}
Proposition~\ref{prop:feas_compact} makes every round well-defined and keeps the iterates in $K$. Lemma~\ref{lem:one_round} proves part~(i). Summing its decrease gives the first bound of Lemma~\ref{lem:R2}; dividing by $R$ proves~\eqref{eq:tier1_rate}. The same bound gives $\sum_r\|\hat z_r\|^2<\infty$, hence $\hat z_r\to0$. Both $y^r$ and $\hat y^{r+1}$ lie in $K$, so $\|\hat z_r\|\leq\operatorname{diam}(K)$. Lemma~\ref{lem:retraction} therefore gives, uniformly in $r$,
\[
  \|y^{r+1}-y^r\|\leq
  \Bigl(1+\tfrac{L_f}{2}\sqrt{|\cS||\cA|}\,\operatorname{diam}(K)\Bigr)\|\hat z_r\|\longrightarrow0.
\]

For part~(iii), take a subsequence $y^{r_j}\to y^\star$. Since $\hat z_{r_j}\to0$, also $\hat y^{r_j+1}\to y^\star$. Validity and the trust inequalities pass to the limit; $y^{r_j}\geq b_{r_j}\to b_\infty$ gives the limiting floor. Thus $y^\star\in\mathcal F_\infty$.

Assume LICQ for the active constraints of $\mathcal F_\infty$ at $y^\star$. Passing to a further subsequence, the active set of the convex subproblem is fixed, since there are finitely many constraints. Every trust constraint active at $\hat y^{r_j+1}$ is active at $y^\star$ by continuity of $D_\alpha^s$. Likewise, $\hat y^{r_j+1}(s,a)=b_{r_j}(s,a)$ for an active floor implies $y^\star(s,a)=b_\infty(s,a)$. Form the matrix $G_j$ from the subproblem's equality gradients $\nabla f_s(y^{r_j})$, its active trust gradients $\nabla D_\alpha^s(\hat y^{r_j+1})$, and its active floor gradients $-e_{(s,a)}$. These columns converge to a subset of the active-gradient matrix of $\mathcal F_\infty$ at $y^\star$. LICQ makes that limit linearly independent, so the smallest singular value of $G_j$ is at least some $\sigma>0$ for all large $j$. In particular, the subproblem itself satisfies LICQ there.

KKT multipliers $\lambda_j$ now exist for the convex subproblem, with nonnegative components for active trust and floor inequalities. Writing $v_j:=\nabla\tilde\Phi_{r_j}(\hat y^{r_j+1})+\rho\hat z_{r_j}$, stationarity is $G_j\lambda_j=-v_j$. Thus $\|\lambda_j\|\leq\|v_j\|/\sigma$. The gradients of $J_k^\pm$ are bounded on $K$, so $v_j$ and the multipliers are bounded. Pass to a convergent multiplier subsequence. As $j\to\infty$, $v_j\to\nabla\Phi(y^\star)$, the tangent-plane gradients become $\nabla f_s(y^\star)$, and feasibility, multiplier signs, and complementary slackness pass to the limit. We obtain the KKT conditions for minimizing $\Phi$ over $\mathcal F_\infty$. Under LICQ its tangent cone equals the linearized feasible cone, so the KKT equation gives $\nabla\Phi(y^\star)^\top d\geq0$ for every tangent direction $d$. Hence $y^\star$ is d-stationary.
\end{proof}

At a fixed point $\hat y^{r+1}=y^r$, the floor stops moving and the subproblem's KKT conditions reduce to those of $\Phi$ over $\mathcal F_\infty$; under LICQ this gives d-stationarity. Part~(ii) bounds the smallest squared step through round $R$ by $O(1/R)$; this alone is not a rate for a KKT residual.


\subsection{Outer-loop setup and assumptions}
\label{app:outer_setup}

The outer-loop analysis (Theorem~\ref{thm:outer}) uses exact expectations and advantages with $k = H$, so $J_H = J(\pi) - J(\pi_n)$, and strict per-state trust regions with $\alpha = 2$ (half the Pearson $\chi^2$ divergence under our normalization). It uses two assumptions and the advantage scale $A_{\max} := \sup_{n,s,a}|A^{\pi_n}(s,a)| \leq 2 Z_H R_{\max}$ (for rewards in $[-R_{\max}, R_{\max}]$), which controls the bias term.

\begin{assumption}[Support floor on outer iterates]
  \label{ass:support}
Every outer iterate satisfies $\pi_n(a|s) \geq p_{\min}$ for all $n, s, a$, with $|\cA|\, p_{\min} < 1$.
\end{assumption}

The moving inner floor guarantees only $\pi_{n+1}(a|s)\geq p_{\min}e^{-B_n(\rho_n)}$ when its nominal floor is $p_{\min}$; repeating this estimate alone does not give one positive uniform floor over infinitely many outer iterations. Assumption~\ref{ass:support} is therefore a separate condition on \emph{all} outer iterates, not a consequence of the initial policy or the inner convergence theorem.  In particular, raising the nominal floor above an action probability of $\pi_n$ would make $y^0=0$ infeasible.  The contraction theorem below is conditional on this explicit assumption.

\begin{assumption}[Concentrability]
  \label{ass:conc}
There is a constant $C_\infty$ such that $d^{\pi^*}(s)\leq C_\infty d^{\pi_n}(s)$ for every $n$ and $s$.
\end{assumption}

Assumption~\ref{ass:support} already implies this with $C_\infty=p_{\min}^{-(H-1)}$: the transition kernels satisfy $P^{\pi_n}(s'|s)\geq p_{\min}P^{\pi^*}(s'|s)$, and induction over time gives $\Pr(S_t=s\mid\pi_n)\geq p_{\min}^t\Pr(S_t=s\mid\pi^*)$. We state Assumption~\ref{ass:conc} separately because the MDP may admit a much smaller $C_\infty$.

At outer iteration $n$, let $\mathcal F_n$ denote the limiting feasible set $\mathcal F_\infty$ of Theorem~\ref{thm:inner}, instantiated with sampling policy $\pi_n$ and $\alpha=2$. It includes policy validity, the strict per-state trust region, and the limiting inner floor $b_{\infty,n}$.

\subsection{Mixture state-visitation dominance}

Throughout the outer-loop proofs we use the un-normalized visitation measure $\eta^\pi(s) = Z_H\, d^\pi(s)$; recall $\sum_s \eta^\pi(s) = Z_H$.

\begin{lemma}[Mixture state-visitation dominance]
  \label{lem:mix_dist}
For any $\lambda \in [0, 1]$ and any policies $\pi_n, \hat\pi$, define $\tilde\pi_\lambda = (1 - \lambda) \pi_n + \lambda \hat\pi$. Then for finite horizon $H$:
  \[
    \eta^{\tilde\pi_\lambda}(s)
    \;\geq\; \sum_{t=0}^{H-1} \gamma^t (1-\lambda)^t\,
      \Pr(S_t = s \mid \pi_n)
    \;\geq\; (1 - \lambda)^{H-1}\, \eta^{\pi_n}(s)
    \qquad \forall s \in \cS.
  \]
\end{lemma}

\begin{proof}
The policy-induced state-transition kernel satisfies $P^{\tilde\pi_\lambda}(s'|s) \geq (1 - \lambda)\, P^{\pi_n}(s'|s)$ entrywise.  Let $\mu_t^\lambda(s) := \Pr(S_t = s \mid \tilde\pi_\lambda)$ and $\mu_t^n(s) := \Pr(S_t = s \mid \pi_n)$.  Induction on $t$ (base $\mu_0^\lambda = \mu_0^n = d_0$; inductive step $\mu_{t+1}^\lambda(s') = \sum_s \mu_t^\lambda(s) P^{\tilde\pi_\lambda}(s'|s) \geq (1-\lambda)^{t+1} \mu_{t+1}^n(s')$) gives $\mu_t^\lambda \geq (1-\lambda)^t \mu_t^n$.  Summing with $\gamma^t$ weights and bounding $(1-\lambda)^t \geq (1-\lambda)^{H-1}$ for $t \leq H - 1$ completes the proof.
\end{proof}

\subsection{Floored greedy comparator}

\begin{definition}[Floored greedy policy]
  \label{def:floored_greedy}
The floored greedy policy at iterate $\pi_n$ is
  \[
    \pi_n^{G, p_{\min}}(a|s) =
    \begin{cases}
      1 - (|\cA| - 1)\, p_{\min}, & a = a_s^* := \argmax_a Q^{\pi_n}(s,a), \\
      p_{\min},                   & a \neq a_s^*.
    \end{cases}
  \]
\end{definition}

\subsection{Global inner optimum lower bound}

This is the conservative-policy-iteration comparison applied to one feasible point of $\mathcal F_n$. At $k=H$, the inner objective is the exact return difference, so its global optimum dominates the comparator's improvement.

\begin{lemma}[Global inner optimum lower bound]
  \label{lem:global_lb}
Under Assumptions~\ref{ass:support}--\ref{ass:conc}, with $\lambda = \min(\sqrt{2\delta/C_G}, 1/2, 1/H)$, $C_G = 1/p_{\min} - 1$, $c = \lambda (1 - |\cA| p_{\min}) (1-\lambda)^{H-1} / C_\infty$ and $\beta = \lambda\, p_{\min}\, |\cA|\, A_{\max}\, Z_H$:
  \[
    -\Phi_n^*
    \;=\; \max_{y \in \mathcal{F}_n}\, [J(\pi) - J(\pi_n)]
    \;\geq\; c\, \Delta_n \;-\; \beta.
  \]
\end{lemma}

\begin{proof}
Consider the mixture $\tilde\pi_\lambda = (1 - \lambda) \pi_n + \lambda\, \pi_n^{G, p_{\min}}$.

\emph{Step 1 (Feasibility).} Both components satisfy the floor ($\pi_n$ by Assumption~\ref{ass:support}, the floored greedy by construction), and the floor is preserved under mixtures: $\tilde\pi_\lambda(a|s) \geq (1-\lambda) p_{\min} + \lambda p_{\min} = p_{\min}$.  For the per-state $\chi^2$ trust region, $\chi^2(\pi_n^{G,p_{\min}}(\cdot|s) \| \pi_n(\cdot|s)) = \sum_a \pi_n^{G,p_{\min}}(a|s)^2 / \pi_n(a|s) - 1 \leq 1/p_{\min} - 1 = C_G$ (using $\sum_a q_a^2 \leq \max_a q_a \leq 1$ for any distribution $q$ and $\pi_n \geq p_{\min}$).  The affine identity $\chi^2((1-\lambda) q + \lambda p \,\|\, q) = \lambda^2 \chi^2(p \| q)$ then gives $D_2^s(\tilde\pi_\lambda \| \pi_n) = \tfrac{\lambda^2}{2}\chi^2(\pi_n^{G,p_{\min}} \| \pi_n) \leq \tfrac{\lambda^2 C_G}{2} \leq \delta$. So $\tilde\pi_\lambda \in \mathcal{F}_n$ (its log-ratio satisfies validity exactly since $\tilde\pi_\lambda$ is a distribution).

\emph{Step 2 (PDL for the mixture).} Using the state-visitation form with the un-normalized measure and $\sum_a \pi_n(a|s) A^{\pi_n}(s, a) = 0$:
  \[
    J(\tilde\pi_\lambda) - J(\pi_n)
    = \sum_s \eta^{\tilde\pi_\lambda}(s)
        \sum_a \tilde\pi_\lambda(a|s)\, A^{\pi_n}(s, a)
    = \lambda \sum_s \eta^{\tilde\pi_\lambda}(s)
        \sum_a \pi_n^{G, p_{\min}}(a|s)\, A^{\pi_n}(s, a).
  \]

\emph{Step 3 (Split the greedy advantage).} By the definition of $\pi_n^{G,p_{\min}}$, the inner sum splits into a nonnegative main term and a signed residual:
  \[
    \sum_a \pi_n^{G,p_{\min}}(a|s) A^{\pi_n}(s,a)
    = (1 - |\cA| p_{\min})\, \Delta^*(s)
      + p_{\min} \Sigma(s),
  \]
Indeed, the greedy action has weight $1-(|\cA|-1)p_{\min}$ and every other action has weight $p_{\min}$. Here $\Delta^*(s):=\max_a A^{\pi_n}(s,a)\geq0$ and $\Sigma(s):=\sum_a A^{\pi_n}(s,a)$, with $|\Sigma(s)|\leq|\cA|A_{\max}$.

\emph{Step 4 (Main term: dominance and concentrability).} Since $\Delta^*(s) \geq 0$, Lemma~\ref{lem:mix_dist} and Assumption~\ref{ass:conc} give
  \[
    \sum_s \eta^{\tilde\pi_\lambda}(s)\, \Delta^*(s)
    \;\geq\; (1-\lambda)^{H-1} \sum_s \eta^{\pi_n}(s)\, \Delta^*(s)
    \;\geq\; \frac{(1-\lambda)^{H-1}}{C_\infty}
       \sum_s \eta^{\pi^*}(s)\, \Delta^*(s).
  \]
Pointwise domination $\Delta^*(s) \geq \sum_a \pi^*(a|s) A^{\pi_n}(s,a)$ and the PDL applied to $\pi^*$ give $\sum_s \eta^{\pi^*}(s) \Delta^*(s) \geq J(\pi^*) - J(\pi_n) = \Delta_n$.

\emph{Step 5 (Residual term: uniform bound).} The residual is weighted by the \emph{mixture} measure $\eta^{\tilde\pi_\lambda}$, which neither the (one-sided) concentrability assumption nor the (lower-bound) dominance lemma can relate to $\eta^{\pi^*}$.  We therefore bound it uniformly:
  \[
    \Bigl| \sum_s \eta^{\tilde\pi_\lambda}(s)\, p_{\min}\, \Sigma(s) \Bigr|
    \;\leq\; p_{\min}\, |\cA|\, A_{\max}
       \sum_s \eta^{\tilde\pi_\lambda}(s)
    \;=\; p_{\min}\, |\cA|\, A_{\max}\, Z_H,
  \]
using only that the total mass of $\eta^{\tilde\pi_\lambda}$ is $Z_H$.  No distribution-shift step is needed or used.

\emph{Step 6 (Collecting).} Combining Steps 2--5:
  \[
    J(\tilde\pi_\lambda) - J(\pi_n)
    \;\geq\;
    \underbrace{\frac{\lambda (1 - |\cA| p_{\min}) (1-\lambda)^{H-1}}{C_\infty}}_{c}
    \, \Delta_n
    \;-\;
    \underbrace{\lambda\, p_{\min}\, |\cA|\, A_{\max}\, Z_H}_{\beta}.
  \]
Since $\tilde\pi_\lambda \in \mathcal{F}_n$ and $-\Phi_n^*$ maximizes $J(\pi) - J(\pi_n)$ over $\mathcal{F}_n$, the bound transfers to $-\Phi_n^*$.
\end{proof}

\subsection{Proof of Theorem~\ref{thm:outer} (Outer-loop contraction)}
\label{app:proof_thm_outer}

Let $\Phi_n$ be the exact full-horizon inner objective at iteration $n$, let $\Phi_n^*:=\min_{y\in\mathcal F_n}\Phi_n(y)$, and let $y_n$ be the returned point. Proposition~\ref{prop:feas_compact} makes $y_n$ feasible for $\mathcal F_n$. Define the global inner-solve gap $\epsilon_n:=\Phi_n(y_n)-\Phi_n^*\geq0$ and $\bar\epsilon:=\sup_n\epsilon_n$. The inner-loop step bound does not bound this global gap.

\begin{proof}
Because $k=H$, the inner objective is the exact return difference. Lemma~\ref{lem:global_lb} therefore gives the per-step improvement
\begin{equation}
  J(\pi_{n+1})-J(\pi_n)
  =-\Phi_n(y_n)=-\Phi_n^*-\epsilon_n
  \geq c\Delta_n-\beta-\epsilon_n.
  \label{eq:per_step}
\end{equation}
Consequently $\Delta_{n+1}\leq(1-c)\Delta_n+\beta+\epsilon_n\leq(1-c)\Delta_n+\beta+\bar\epsilon$. Iterating this recursion yields
\[
  \Delta_N\leq(1-c)^N\Delta_0
    +(\beta+\bar\epsilon)\sum_{j=0}^{N-1}(1-c)^j
  \leq(1-c)^N\Delta_0+\frac{\beta+\bar\epsilon}{c},
\]
which is Theorem~\ref{thm:outer}.
\end{proof}

\begin{remark}[On the uniform bound in Step 5]
  \label{rem:step5_corrected}
The floor forces the comparator to place mass $p_{\min}$ on every action, including those of negative advantage, producing the residual $p_{\min} \Sigma(s)$ weighted by the \emph{mixture} occupancy $\eta^{\tilde\pi_\lambda}$.  That residual cannot be re-weighted to $\eta^{\pi^*}$: Assumption~\ref{ass:conc} bounds $\eta^{\pi^*}/\eta^{\pi_n}$ from above while Lemma~\ref{lem:mix_dist} bounds $\eta^{\tilde\pi_\lambda}$ from \emph{below}, so neither controls $\eta^{\tilde\pi_\lambda}$ from above by $\eta^{\pi^*}$; a state rarely visited by $\pi^*$ but heavily visited by the mixture would break such an inequality.  We therefore bound it uniformly, using only that $\eta^{\tilde\pi_\lambda}$ has total mass $Z_H$.  The price is the worst-case product $|\cA|\, A_{\max}\, Z_H$ in $\beta$, which still vanishes linearly as $p_{\min} \to 0$.
\end{remark}

\begin{remark}[Status of Theorem~\ref{thm:outer}]
  \label{rem:outer_status}
This is a \emph{conditional contraction}. Theorem~\ref{thm:tier1} drives the inner proximal step to zero, but does not bound the gap $\epsilon_n$ to the global minimum of the nonconvex per-iteration program. The bound is informative when $\bar\epsilon$ is small; multi-start SCP can help assess that gap empirically. The neighborhood $(\beta+\bar\epsilon)/c$ also contains the floor bias $\beta$, which vanishes as $p_{\min}\to0$. This comes with an inner-loop cost: the box $K$ grows as $p_{\min}\to0$, and its constants $L_f,L_\Phi,\rho_0$ and the budget $B(\rho)=L_f[\Phi(y^0)-\inf_K\Phi]/(\rho-\rho_0)$ can worsen. Keeping $B(\rho)\leq1$ may then require a larger $\rho$.
\end{remark}

\begin{remark}[Horizon dependence]
  \label{rem:H_dependence}
The factor $\lambda(1-\lambda)^{H-1}$ in $c$ decays exponentially when $\lambda$ is held fixed. Our choice $\lambda=\min(\sqrt{2\delta/C_G},1/2,1/H)$ caps it at $1/H$. If $\delta\geq C_G/(2H^2)$, then for $H\geq2$ we have $\lambda=1/H$ and $\lambda(1-\lambda)^{H-1}\geq1/(eH)$, so $c\geq(1-|\cA|p_{\min})/(eH C_\infty)$. For a smaller trust radius, $\lambda=\sqrt{2\delta/C_G}$ may give a smaller constant. The remaining horizon dependence lies in $C_\infty$ and the inner-loop constants.
\end{remark}

\section{Additional Experimental Results}
\label{app:experiments}

This appendix reports the supporting analyses and full experimental setups referenced in Section~\ref{sec:experiments}.

\subsection{Risky Corridor MDP: advantage under the uniform behavior policy}
\label{app:risky-corridor-advantage}

We derive $A^{\pi_b}(s_0, L)$ used in Section~\ref{sec:risky-corridor}.  The Risky Corridor MDP has horizon $H = 2$, $\gamma = 1$, and deterministic transitions (Figure~\ref{fig:comparison}(a)).  Under a uniform behavior policy $\pi_b$:
\begin{align*}
  V^{\pi_b}(s_R) &= 1 \quad \text{(single action, terminal)}, \\
  V^{\pi_b}(s_L) &= \tfrac{1}{2}(100) + \tfrac{1}{2}(-80) = 10, \\
  V^{\pi_b}(s_0) &= \tfrac{1}{2}\bigl(-20 + V^{\pi_b}(s_L)\bigr)
                 + \tfrac{1}{2}\bigl(3 + V^{\pi_b}(s_R)\bigr) \\
                 &= \tfrac{1}{2}(-10) + \tfrac{1}{2}(4) = -3.
\end{align*}
The Q-values at $s_0$ are $Q^{\pi_b}(s_0, L) = -20 + V^{\pi_b}(s_L) = -10$ and $Q^{\pi_b}(s_0, R) = 3 + V^{\pi_b}(s_R) = 4$, giving
\[
  A^{\pi_b}(s_0, L) = -10 - (-3) = -7,
  \qquad
  A^{\pi_b}(s_0, R) = 4 - (-3) = 7.
\]
The single-step advantage of the optimal action $L$ is therefore negative under $\pi_b$, which is what makes the MDP a stress test for the CPI surrogate at $k = 1$.

\subsection{Risky Corridor MDP: detailed results}
\label{app:risky-corridor-table}

Table~\ref{tab:risky_corridor} reports the quantitative results corresponding to the convergence experiment in Section~\ref{sec:risky-corridor}: area under the normalized return curve (AUC), final normalized return, and steps to reach $90\%$ of optimal performance ($t_{90}$).  All four methods use a common budget of $32$ episodes per update over $10$ updates, and each is tuned by the same two-stage protocol. We first evaluate every configuration in the method's grid ($900$, $432$, $432$, and $288$ configurations for PPO, CCRL, AWR, and TRPO) on $10$ tuning seeds, rank by AUC with final return and $t_{90}$ as tie-breakers, and retain the top five.  These five are re-evaluated on $100$ validation seeds, disjoint from the tuning seeds, and the best of the five is reported.  Means and standard errors below are over those $100$ validation seeds.

\begin{table}[H]
    \centering
    \begin{tabular}{l r@{\,$\pm$\,}l r@{\,$\pm$\,}l r@{\,$\pm$\,}l}
        \toprule
        Method & \multicolumn{2}{c}{AUC} & \multicolumn{2}{c}{Final return} & \multicolumn{2}{c}{$t_{90}$} \\
        \midrule
        PPO  & 8.1951 & 0.0509 & 0.9662 & 0.0056 & 3.29 & 0.09 \\
        TRPO & 8.4924 & 0.0261 & \textbf{0.9996} & 0.0002 & 3.13 & 0.04 \\
        AWR  & 8.9028 & 0.0817 & 0.9908 & 0.0091 & 2.09 & 0.09 \\
        CCRL & \textbf{9.2057} & 0.0882 & 0.9908 & 0.0091 & \textbf{1.38} & 0.11 \\
        \bottomrule
    \end{tabular}
    \caption{Risky Corridor MDP results across PPO, TRPO, AWR, and
    CCRL: AUC of the normalized return curve, final normalized
    return, and steps-to-$90\%$-of-optimal ($t_{90}$).  Mean
    $\pm$ SE over $100$ validation seeds.  Best mean per metric in bold where
    unique.}
    \label{tab:risky_corridor}
\end{table}

\subsection{Advantage-estimation quality on CartPole}
\label{app:adv-quality}

This appendix details the advantage-estimation quality experiment referenced in Section~\ref{sec:classic-control}; the results are summarized in Figure~\ref{fig:adv-quality}.

\begin{figure}[h]
    \centering
    \includegraphics[width=\linewidth]{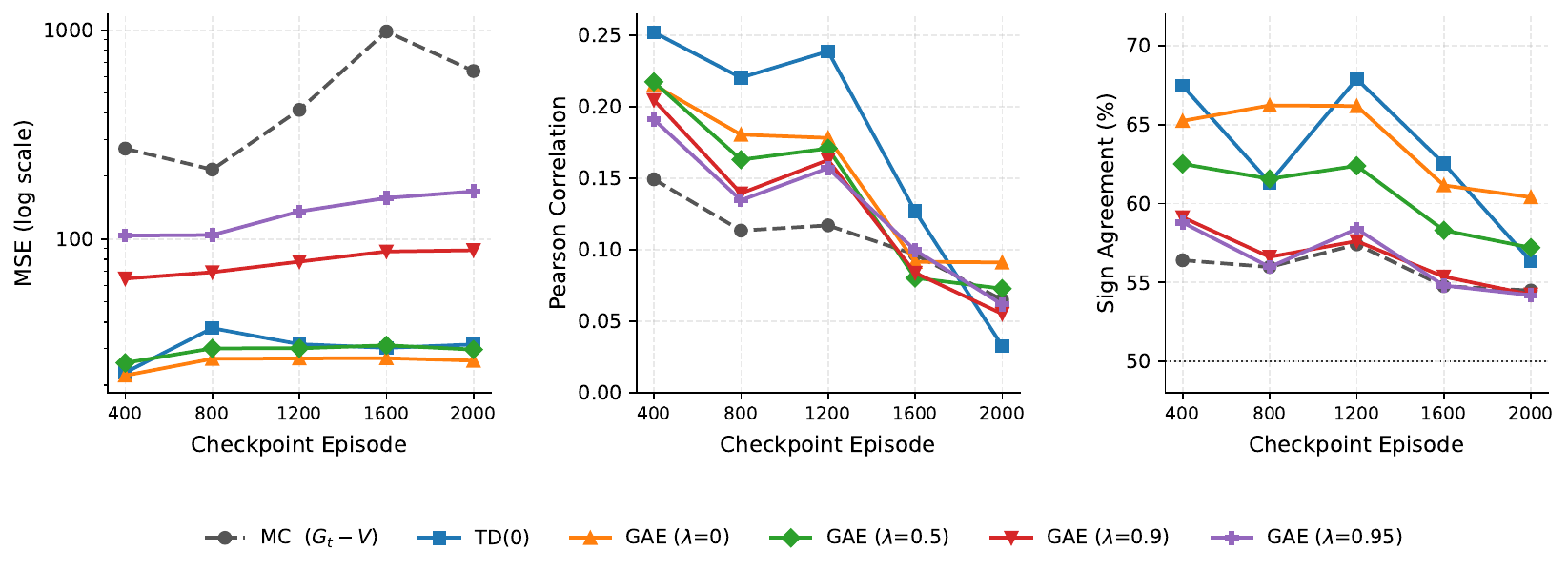}
    \caption{\textbf{Advantage-estimation quality across training
    checkpoints (CartPole).}  Six advantage estimators---a
    Monte-Carlo baseline ($G_t - V$), TD(0), and GAE at
    $\lambda \in \{0, 0.5, 0.9, 0.95\}$---are compared against a
    Monte-Carlo ground-truth advantage at five policy checkpoints
    (x-axis), along MSE on a log scale (left), Pearson correlation
    (center), and sign agreement (right).  Dotted reference lines
    mark zero correlation and $50\%$ (chance) sign agreement.
    No single estimator dominates: the MC baseline and high-$\lambda$
    GAE ($0.9, 0.95$) incur the largest MSE by an order of magnitude,
    while TD(0) and low-$\lambda$ GAE ($\lambda{=}0$) achieve both the
    lowest MSE and the highest correlation and sign agreement; the GAE
    variants interpolate between these extremes as $\lambda$ grows.
    Across all estimators, correlation and sign agreement decay toward
    their no-information values (zero correlation, $50\%$ sign
    agreement) as the policy improves.}
    \label{fig:adv-quality}
\end{figure}

\paragraph{Environment and policy.}
We use CartPole (horizon $H = 200$, undiscounted $\gamma = 1$, $+1$ per surviving step) with a softmax policy parameterized by a $4{\to}32{\to}32{\to}2$ tanh MLP.  The policy is trained for $2000$ episodes by REINFORCE with return standardization, batch size $16$, and Adam at learning rate $10^{-3}$, saving a checkpoint every $400$ episodes.  This yields five checkpoints whose average return rises from $24.5$ to $60.4$, letting us track estimation quality as the policy improves.

\paragraph{Ground-truth advantage.}
At each checkpoint we freeze the policy $\pi_n$ and collect $50$ evaluation trajectories.  For every visited state-action pair $(s,a)$ we form a near-unbiased target $A^{\pi_n}(s,a) = Q^{\pi_n}(s,a) - V^{\pi_n}(s)$ by Monte-Carlo rollouts under $\pi_n$: $V^{\pi_n}(s)$ averages the return of $50$ rollouts from $s$, and $Q^{\pi_n}(s,a)$ averages $50$ rollouts that first take $a$ and then follow $\pi_n$.  All estimators are scored against this target.

\paragraph{Estimators.}
We evaluate six estimators, each built on a learned value critic (a $4{\to}64{\to}64{\to}1$ tanh MLP trained $300$ epochs with Adam at $10^{-3}$).  Two distinct critics are fit, and each estimator uses the one conventional for it:
\begin{itemize}[leftmargin=2em,itemsep=1pt]
\item the \textbf{TD(0) critic} is trained by minimizing the squared TD error $\bigl(V(S) - [\,R + V(S')\,]\bigr)^2$ (one-step bootstrapping, with a stop-gradient on the target $V(S')$);
\item the \textbf{MC critic} is trained by regression onto Monte-Carlo return samples, i.e.\ minimizing $(V(S_t) - G_t)^2$ against the empirical return-to-go $G_t$.
\end{itemize}
The estimators are then:
\begin{itemize}[leftmargin=2em,itemsep=1pt]
\item \textbf{MC baseline:} $\widehat{A}(S_t,A_t) = G_t - V(S_t)$, with $G_t$ the empirical return-to-go and $V$ the MC critic;
\item \textbf{TD(0):} $\widehat{A}(S_t,A_t) = R_{t+1} + V(S_{t+1}) - V(S_t)$ using the TD(0) critic;
\item \textbf{GAE($\lambda$)} for $\lambda \in \{0, 0.5, 0.9, 0.95\}$: the exponentially-weighted sum $\widehat{A}_t = \sum_{l \ge 0} \lambda^l \delta_{t+l}$, computed by the backward recursion $\widehat{A}_t = \delta_t + \lambda\, \widehat{A}_{t+1}$ with one-step TD residual $\delta_t = R_{t+1} + V(S_{t+1}) - V(S_t)$ on the MC critic (terms truncated at episode end).
\end{itemize}
Thus the TD(0) advantage draws on the TD-error-trained critic, while the MC baseline and every GAE variant draw on the MC-return-trained critic. Since $\gamma = 1$, the discount is dropped throughout; GAE($0$) reduces to the one-step residual on the MC critic.

\paragraph{Metrics.}
We report three metrics, each computed over state-action pairs with a non-negligible true advantage ($|A^{\pi_n}| > 0.5$): Pearson correlation between estimated and true advantage, sign agreement (the fraction of pairs whose estimated advantage has the correct sign), and mean-squared error (MSE).

\paragraph{Findings.}
Figure~\ref{fig:adv-quality} reports all three metrics against training progress.  Two patterns stand out.  First, no estimator is uniformly best: low-$\lambda$ GAE and TD(0) minimize MSE ($\approx 22\text{--}37$) but the Monte-Carlo baseline and high-$\lambda$ GAE incur an order-of-magnitude larger MSE ($\sim\!10^2\text{--}10^3$).  Second, and most relevant to CCRL, correlation and sign agreement are modest and \emph{deteriorate as the policy improves}: TD(0) correlation falls from $0.25$ to $0.03$ and sign agreement drifts from $\sim\!68\%$ toward the $50\%$ chance line.  Because the DC decomposition (Theorem~\ref{thm:dc}) assigns terms to $J_k^+$ / $J_k^-$ by advantage sign, this sign-error rate directly sets how far the optimized manifold departs from the true one.

\subsection{PDIS-$k$ objective variance on CartPole}
\label{app:pdis-variance}

This appendix details the variance experiment referenced in Section~\ref{sec:classic-control}; the results are summarized in Figure~\ref{fig:pdis-variance}.  The experiment isolates the effect of the truncation depth $k$ on the variance of the per-trajectory objective estimator, holding the advantages fixed so that variance differences are attributable to the $k$-step importance-sampling window alone.

\begin{figure}[h]
    \centering
    \includegraphics[width=\linewidth]{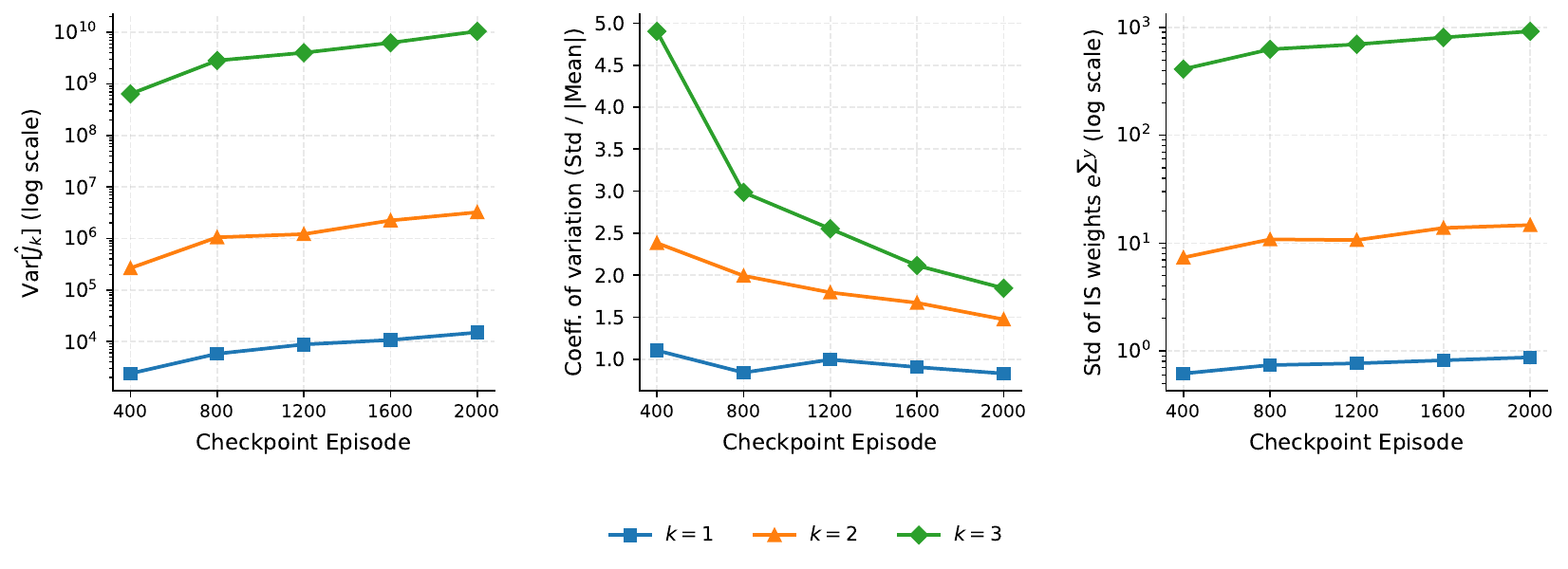}
    \caption{\textbf{PDIS-$k$ objective variance grows with the
    truncation depth $k$ (CartPole).}  At five policy checkpoints
    (x-axis), the per-trajectory objective $\hat{J}_k(\tau)$ is
    evaluated at the SCP-optimized log-ratios for
    $k \in \{1, 2, 3\}$ over $200$ trajectories, with the
    advantages held fixed at their Monte-Carlo ground-truth values.
    Left: variance of $\hat{J}_k$ across trajectories (log scale);
    center: its coefficient of variation (std over $|$mean$|$);
    right: the standard deviation of the cumulative importance-sampling
    weights $e^{\sum y}$ (log scale).  Variance is ordered
    $k{=}1 < k{=}2 < k{=}3$ throughout and separates by two to three
    orders of magnitude per increment of $k$, driven by the widening
    spread of the IS weights.}
    \label{fig:pdis-variance}
\end{figure}

\paragraph{Environment, policy, and checkpoints.}
We use the same CartPole setup and REINFORCE-trained policy as in Appendix~\ref{app:adv-quality} ($4{\to}32{\to}32{\to}2$ tanh softmax policy, $2000$ episodes, checkpoint every $400$), giving five checkpoints with average return rising from $24.5$ to $62.3$.  Here we discount with $\gamma = 0.99$.

\paragraph{Protocol.}
At each checkpoint we freeze the policy and collect $N = 200$ trajectories.  For every visited pair we compute the near-exact Monte-Carlo advantage $A^{\pi_n}(s,a) = Q^{\pi_n}(s,a) - V^{\pi_n}(s)$ ($50$ rollouts each for $Q^{\pi_n}$ and $V^{\pi_n}$), exactly as in Appendix~\ref{app:adv-quality}.  We then solve the CCRL inner problem~\eqref{eq:ccrl_full} by SCP ($4$ rounds of $50$ inner primal-dual iterations, step size $0.1$, per-state $\chi^2$ trust region with $\alpha = 2$, $\delta = 0.2$, log-ratios clipped to $|y| \le 3$) separately for each $k \in \{1, 2, 3\}$, obtaining optimized log-ratios $y(s,a)$.  Crucially, the advantages are shared across all three solves; only the $k$-window in the objective changes.

\paragraph{Variance measure.}
For each $k$ and each trajectory $\tau_i$ we evaluate the per-trajectory objective
\begin{equation}
  \hat{J}_k(\tau_i)
  = \sum_{t=0}^{L_i - 1} \gamma^t\,
      \exp\!\Bigl(\textstyle\sum_{t'=\max(0,\,t-k+1)}^{t} y(s_{t'}, a_{t'})\Bigr)\,
      A^{\pi_n}(s_t, a_t),
  \label{eq:per_traj_Jk}
\end{equation}
and report its variance and coefficient of variation (std over $|$mean$|$) across the $N$ trajectories, together with the standard deviation of the cumulative importance-sampling weights $e^{\sum y}$ that appear inside~\eqref{eq:per_traj_Jk}.

\paragraph{Findings.}
Figure~\ref{fig:pdis-variance} reports all three quantities against training progress.  The estimator variance is ordered $k{=}1 < k{=}2 < k{=}3$ at every checkpoint, separating by roughly two to three orders of magnitude per unit increase in $k$ (e.g.\ at the final checkpoint, $\mathrm{Var}[\hat{J}_1] \approx 10^4$, $\mathrm{Var}[\hat{J}_2] \approx 10^6$, and $\mathrm{Var}[\hat{J}_3] \approx 10^{10}$).  The coefficient of variation is likewise ordered by $k$, and the right panel shows the mechanism directly: the spread of the IS weights $e^{\sum y}$ grows sharply with $k$ (the maximum weight rises from $\sim\!20$ at $k=1$ to $\sim\!8000$ at $k=3$), since the exponential of a $k$-term sum of log-ratios amplifies trajectory-to-trajectory variation.  This is the bias--variance tension of the truncation depth: larger $k$ couples credit across more decisions but at rapidly growing estimator variance.

\subsection{Risky Chain: hyperparameter grids and tuning protocol}
\label{app:chain-hparams}

This appendix gives the full tuning setup behind the Risky Chain results of Section~\ref{sec:risky-chain}, so that the fairness of the PPO comparison can be checked directly.  The MDP is illustrated in Figure~\ref{fig:risky-chain}.

\begin{figure}[h]
    \centering
    \resizebox{0.7\textwidth}{!}{%
    \begin{tikzpicture}[
        >={Stealth[length=2.5mm]},
        state/.style={circle, draw, minimum size=10mm, inner sep=0pt, font=\small},
        terminal/.style={circle, draw, double, minimum size=10mm, inner sep=0pt, font=\small, fill=yellow!15},
        absorb/.style={rectangle, draw, rounded corners=1pt, minimum width=8mm, minimum height=5mm, font=\scriptsize, fill=gray!10},
        reward/.style={rectangle, draw, rounded corners=1pt, minimum width=8mm, minimum height=5mm, font=\scriptsize, fill=green!15},
        penalty/.style={rectangle, draw, rounded corners=1pt, minimum width=8mm, minimum height=5mm, font=\scriptsize, fill=red!15},
        edgelbl/.style={font=\scriptsize, inner sep=1pt},
        node distance=14mm and 18mm,
    ]
        \node[state] (s0) {$s_0$};
        \node[state, right=of s0] (s1) {$s_1$};
        \node[right=of s1] (dots) {$\cdots$};
        \node[state, right=of dots] (sn2) {$s_{N-2}$};
        \node[terminal, right=of sn2] (sn1) {$s_{N-1}$};
        \draw[->] (s0) -- node[edgelbl, above]{continue, $-5$} (s1);
        \draw[->] (s1) -- node[edgelbl, above]{continue, $-5$} (dots);
        \draw[->] (dots) -- node[edgelbl, above]{continue, $-5$} (sn2);
        \draw[->] (sn2) -- node[edgelbl, above]{continue, $-5$} (sn1);
        \node[absorb, below=8mm of s0]   (e0) {exit, $+1$};
        \node[absorb, below=8mm of s1]   (e1) {exit, $+1$};
        \node[absorb, below=8mm of sn2]  (en2) {exit, $+1$};
        \draw[->] (s0) -- (e0);
        \draw[->] (s1) -- (e1);
        \draw[->] (sn2) -- (en2);
        \node[reward,  above right=4mm and 6mm of sn1]  (aplus)  {$a^+, +100$};
        \node[penalty, below right=4mm and 6mm of sn1]  (aminus) {$a^-, -80$};
        \draw[->] (sn1) -- (aplus);
        \draw[->] (sn1) -- (aminus);
    \end{tikzpicture}%
    }%
    \caption{Risky chain MDP with $N$ stages ($\gamma = 1$).  Each
    intermediate state $s_i$ ($i < N{-}1$) offers a safe-exit action
    (reward $+1$, terminal) or a continue action (reward $-5$,
    advance).  The terminal state $s_{N-1}$ has two actions: $a^+$
    (reward $+100$) and $a^-$ (reward $-80$).  The optimal trajectory
    is $s_0 \to s_1 \to \cdots \to s_{N-1} \to a^+$, with return
    $100 - 5(N-1)$.}
    \label{fig:risky-chain}
\end{figure}
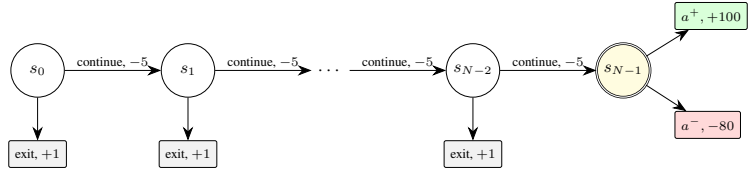

\paragraph{Networks.}
The policy and value are each a one-hidden-layer MLP ($\tanh$ activation; $16$ or $32$ hidden units, selected by tuning) mapping a one-hot state encoding to action logits and to a scalar value, respectively.  PPO and CCRL use the identical architecture, so any performance difference is attributable to the policy-improvement operator alone.

\paragraph{Two-stage protocol.}
Each (algorithm, $N$, $k$) cell is tuned independently; nothing is shared or transferred across cells.  The per-update data budget is fixed at $128$ trajectories and the number of policy-improvement iterations at $200$ for every arm.  \emph{Stage~1 (selection).} We evaluate the entire grid on a fixed set of $10$ tuning seeds $\{7,11,19,23,31,37,41,43,47,53\}$ and rank configurations by nAUC (with final return and then $-t_{90}$ as tie-breakers).  We rank by nAUC rather than final return because the per-seed final return is near-bimodal on this task---a seed either solves the chain or collapses to early exit---so nAUC is the lower-variance selection signal.  We keep the top $12$ configurations. \emph{Stage~2 (validation).} The top $12$ are re-evaluated on $300$ fresh validation seeds (seed base $300{,}000$, disjoint from the tuning seeds), and the configuration with the best mean final return is reported.  All numbers in Table~\ref{tab:chain-mlp} are Stage-2 validation numbers.

\paragraph{Search grids.}
Table~\ref{tab:chain-grids} lists the grids.  The entropy-bonus coefficient uses the \emph{same} six values for both algorithms, and both tune the hidden width, so PPO is never disadvantaged on exploration or capacity.  The PPO grid spans $1296$ configurations per cell and the CCRL grid $3456$; deeper chains require the larger entropy values to reliably reach the optimum under either algorithm.

\begin{table}[H]
\centering
\footnotesize
\setlength{\tabcolsep}{6pt}
\renewcommand{\arraystretch}{1.15}
\caption{Hyperparameter grids for the Risky Chain study (batch fixed at $128$,
$200$ iterations, $\gamma=1$).  The entropy grid is shared across methods.}
\label{tab:chain-grids}
\begin{tabular}{@{}l l@{}}
\toprule
\multicolumn{2}{@{}l}{\textbf{PPO} ($1296$ configs/cell)} \\
\midrule
hidden width & $\{16, 32\}$ \\
policy learning rate & $\{0.01, 0.05, 0.1\}$ \\
critic learning rate & $\{0.05, 0.1\}$ \\
epochs per update & $\{10, 20, 40\}$ \\
clip $\varepsilon$ & $\{0.1, 0.2, 0.3\}$ \\
GAE-$\lambda$ & $\{0.95, 1.0\}$ \\
entropy-bonus coef. & $\{0, 7, 15, 30, 50, 80\}$ \\
\midrule
\multicolumn{2}{@{}l}{\textbf{CCRL} ($3456$ configs/cell)} \\
\midrule
hidden width & $\{16, 32\}$ \\
policy (M-step) learning rate & $\{0.05, 0.1\}$ \\
M-step epochs & $\{40, 80\}$ \\
divergence order $\alpha$ & $\{1.5, 2.0, 3.0\}$ \\
trust radius $\delta$ & $\{0.05, 0.1, 0.2\}$ \\
SCP rounds & $\{5, 8\}$ \\
SCP inner iterations & $\{120, 240\}$ \\
SCP inner learning rate & $\{0.01, 0.03\}$ \\
GAE-$\lambda$ & $\{0.95, 1.0\}$ \\
entropy-bonus coef. & $\{0, 7, 15, 30, 50, 80\}$ \\
\bottomrule
\end{tabular}
\end{table}

\subsection{Therapeutic-Threshold Sepsis MDP: environment and experimental setup}
\label{app:sepsis}

This appendix specifies the Therapeutic-Threshold Sepsis MDP of Section~\ref{sec:sepsis} and the protocol used for the PPO/CCRL comparison reported there.

\paragraph{Environment.}
A state is a pair $(h,d)$: illness severity $h \in \{0,\dots,M-1\}$ and accumulated dose level $d \in \{0,\dots,D-1\}$, giving $|\cS| = M\!\cdot\!D$ states.  Severity $h{=}0$ (recovered) and $h{=}M-1$ (dead) are absorbing. Every episode starts deterministically at mid-severity with zero dose, $(h,d) = (\lfloor (M-1)/2 \rfloor,\, 0)$, and runs for a fixed horizon $H = 20$.  Three actions are available---\textsf{rest}, \textsf{dose}, and \textsf{administer}---with the following transitions from a non-terminal $(h,d)$ (severity moves by at most one level per step):
\begin{itemize}
\item \textsf{rest}: $d$ is unchanged; $h \to h+1$ with probability $p_{\mathrm{rest}} = 0.40$, otherwise unchanged.
\item \textsf{dose}: $d \to \min(d+1,\, D-1)$; $h \to h+1$ with probability $p_{\mathrm{dose}} = 0.20$, otherwise unchanged.
\item \textsf{administer}: the dose resets, $d \to 0$, and \emph{at threshold} ($d = D-1$) the patient recovers, $h \to h-1$, with probability $p_{\mathrm{rec}} = 0.90$ (else unchanged), while \emph{below threshold} ($d < D-1$) the patient is harmed, $h \to h+1$, with probability $p_{\mathrm{harm}} = 0.70$ (else unchanged).
\end{itemize}
A reward of $+1$ is delivered on any transition into $h{=}0$ (recovery) and $0$ on every other transition, including into $h{=}M-1$ (death); the reward is thus sparse and terminal, and the undiscounted expected return equals the survival probability.  We use $M = 5$ severity levels and $D = 3$ dose levels, for which the finite-horizon optimal policy---computed exactly by dynamic programming---survives $0.840$ of episodes. From $d=0$, two \textsf{dose} actions reach the threshold $d=2$; only a later \textsf{administer} action can then improve severity. Thus an action that enables recovery can precede the recovery reward by several decisions.

\paragraph{Algorithms and training.}
Both PPO and CCRL use a parametric MLP policy and critic with a one-hot state encoding, fit a GAE value baseline, and train for $500$ updates at a matched per-update budget.  CCRL solves the per-iteration DC program~\eqref{eq:ccrl_full} by SCP with a per-state $\alpha$-divergence trust region and the $y$-clip of Section~\ref{sec:practical}; PPO uses the standard clipped surrogate.  The $k$-step objective and the clipped surrogate each carry a variance-control knob---CCRL's clip $y_{\max}$ and PPO's clip $\varepsilon$---so we sweep that knob for each method and, at every value, sub-tune the remaining hyperparameters (learning rates, trust-region and inner-loop settings, entropy) so that each reported point is the best that setting attains.  Configurations are evaluated over $100$ seeds, and we report the final survival rate, the area under the survival curve (AUC), and $t_{90}$ (updates to first reach $90\%$ of the dynamic-programming optimum, i.e.\ $0.756$).

\paragraph{Hyperparameter tuning.}
Each cell is tuned independently by the same two-stage protocol used for the Risky Chain (Appendix~\ref{app:chain-hparams}), at a fixed budget of $128$ episodes per update; nothing is shared or warm-started across cells. \emph{Stage~1 (selection).}  We evaluate the method's entire grid on $10$ tuning seeds $\{7,11,19,23,31,37,41,43,47,53\}$ and rank configurations by AUC.  The PPO grid sweeps hidden width $\{16,32\}$, policy learning rate $\{0.01,0.05,0.1\}$, epochs per update $\{10,20,40\}$, clip $\varepsilon \in \{0.1,0.2,0.3\}$, discount $\{0.99,1.0\}$, critic learning rate $\{0.05,0.1\}$, and entropy coefficient $\{0,0.02,0.05,0.1,0.2,0.5\}$, giving $1296$ configurations per cell.  The CCRL grid sweeps hidden width $\{16,32\}$, policy learning rate $\{0.05,0.1\}$, policy-fit projection steps $\{40,80\}$, divergence order $\alpha \in \{1.5,2.0,3.0\}$, trust-region budget $\delta \in \{0.05,0.1,0.2\}$, inner step size $\{0.01,0.03\}$, SCP rounds $\{5,8\}$, inner primal-dual iterations $\{120,240\}$, and the same entropy grid, giving $3456$ configurations per cell.  \emph{Stage~2 (validation).}  The top $12$ configurations are re-evaluated on $100$ validation seeds disjoint from the tuning seeds, and the best of those is reported.  The entropy coefficient is swept on the same grid for both methods, so PPO is never disadvantaged on exploration or capacity.

\subsection{Classic control: environment and experimental setup}
\label{app:classic-control}

This appendix details the classic-control comparison of Section~\ref{sec:classic-control}; the main-text figure is Figure~\ref{fig:control}.  We use four standard discrete-action Gymnasium environments---CartPole-v1, Acrobot-v1, MountainCar-v0, and LunarLander-v3---and compare a tuned PPO baseline against CCRL at truncation depths $k \in \{1, 2, 3\}$.  Each of the sixteen (environment, algorithm) cells is independently hyperparameter-tuned; the selected configurations are then run over $100$ seeds (seeds $0$--$99$) for $200$ policy iterations.  For CCRL, advantages are estimated by GAE and the per-iteration DC program is solved by the tangent-plane SCP of Section~\ref{sec:algorithm}; PPO uses the standard clipped surrogate.  We report the per-seed learning curve (return vs.\ iteration), the final return (average over the last $10$ iterations), and the area under the learning curve (AUC, mean return per iteration), aggregated as mean $\pm$ standard error over seeds.

\section{Extended related work}
\label{app:related}

\textbf{Iterative on-policy policy optimization.}  A large family of methods improves a policy by repeatedly maximizing a per-iteration surrogate of $J(\pi)$ under a proximity constraint to the current iterate.  Conservative Policy Iteration (CPI) \citep{kakade2002cpi} established the surrogate-based template. NPG uses natural-gradient updates \citep{kakade2002npg}, TRPO imposes a KL constraint \citep{schulman2015trpo}, and PPO clips policy ratios \citep{schulman2017ppo}. Off-policy advantage-weighted regression---AWR \citep{peng2020awr} and AWAC \citep{nair2020awac}---uses regression toward an exponentially tilted behavior policy, $\pi^*(a|s) \propto \pib(a|s)\exp(A^{\pib}(s,a)/\eta)$.  These methods use related single-step advantage-weighted objectives, while AWR and AWAC also reuse off-policy data. CCRL recovers the CPI surrogate and the constrained programs underlying TRPO and AWR, with NPG arising through a local approximation (Section~\ref{sec:special_cases}); PPO's clipping heuristic and off-policy data reuse are outside this constrained-program recovery. CCRL extends the objective along the multi-step axis $k \geq 2$.

\textbf{Multi-step credit assignment.}  The difficulty of propagating credit across delayed rewards is classically addressed by injecting multi-step information into the \emph{value target}: TD($\lambda$) and eligibility traces \citep{sutton1988td} interpolate between one-step bootstrapping and Monte-Carlo returns, generalized advantage estimation (GAE) \citep{schulman2016gae} applies the same exponential weighting to advantage estimates for variance reduction, and off-policy return operators such as Retrace($\lambda$) \citep{munos2016retrace} correct for the behavior--target mismatch with truncated importance weights.  These methods shape the advantage estimates fed into a single-step update; CCRL is complementary---it places the multi-step horizon $k$ inside the \emph{optimization objective itself}~\eqref{eq:pdis_k}, so the policy update, not just the critic target, couples consecutive decisions.

\textbf{Importance sampling as an optimization objective.}  PDIS and its variants \citep{precup2000eligibility, jiang2016doubly} are primarily studied for off-policy \emph{evaluation}.  A closer line uses importance sampling as an \emph{optimization} objective: POIS \citep{metelli2018policy} reuses collected trajectories by optimizing an IS-based surrogate offline, with a high-confidence penalty controlling the estimator's variance.  We likewise treat PDIS as the optimization target rather than an evaluation estimator, but expose the surrogate's difference-of-convex structure in log-density-ratio coordinates and use the truncation depth $k$ to trade off multi-step credit propagation against variance.

\textbf{Convergence of policy-gradient methods.} Tabular policy-gradient convergence has been established at several rates.  \citet{mei2020softmax} give $O(1/t)$ global convergence for softmax PG via a non-uniform {\L}ojasiewicz inequality, though \citet{li2023softmax} show that worst-case iteration counts can grow exponentially with problem size and effective horizon.  \citet{agarwal2021theory} give $O(1/\varepsilon)$ suboptimality for natural policy gradient (NPG) with tabular softmax policies and exact gradients.  \citet{cen2022fast} obtain linear convergence for entropy-regularized NPG, and \citet{xiao2022convergence} sharpens this to linear convergence for policy mirror descent (PMD) with geometrically increasing step sizes (subsuming NPG) \emph{without} entropy regularization.  This line has since been generalized: \citet{lan2023pmd} establishes linear convergence and sampling-complexity bounds for PMD across a broad class of regularized MDPs, and \citet{alfano2023pmd} extends linear-convergence guarantees to PMD with general (including neural) parameterizations.  These results analyze particular policy-gradient and mirror-descent updates, including methods that optimize the non-convex softmax return directly. General non-convex policy-gradient analysis gives stationarity guarantees \citep{agarwal2022rlbook}; our analysis instead applies SCP to the $k$-step PDIS objective, which is nonlinear in $\pi$ for $k>1$.  For any $k$, our inner-loop analysis (Section~\ref{sec:convergence}) bounds the smallest squared SCP step among the first $R$ rounds by $O(1/R)$ under the theorem's conditions and gives d-stationarity of limit points satisfying LICQ. At $k=H$, the outer analysis gives conditional contraction of the performance gap toward an error neighborhood when the inner solution is near-global.

\textbf{Difference-of-convex and geometric programming.}  The change of variable that convexifies our objective---taking logarithms to turn multiplicative importance-ratio products into additive exponentials---is the defining move of \emph{geometric programming}, in which posynomial programs become convex in log-transformed variables \citep{boyd2007gp}.  Applied to the $k$-step PDIS objective, this exposes a difference-of-convex (DC) structure.  DC programming has a long history \citep{tao1997convex, lipp2016variations, lethi2018dc} but has seen limited use in reinforcement learning; our contribution is to identify the DC structure in the PDIS objective and to handle the policy-simplex validity constraint---a DC equality---within our SCP procedure (Section~\ref{sec:algorithm}).  These guarantees concern the theoretical algorithm under the conditions of Section~\ref{sec:convergence}; full-sequence rates remain open.

\end{document}